\documentclass[a4paper]{article}
\usepackage[margin=1in]{geometry}
\usepackage{graphicx}%
\usepackage{multirow}%
\usepackage{amsmath,amssymb,amsfonts,bbm}%
\usepackage{amsthm}%
\usepackage{mathrsfs}%
\usepackage{bbm}
\usepackage[title]{appendix}%
\usepackage{xcolor}%
\usepackage{textcomp}%
\usepackage{manyfoot}%
\usepackage{booktabs}%
\usepackage{algorithm}%
\usepackage{algorithmicx}%
\usepackage{algpseudocode}%
\usepackage{listings}%
\usepackage{hyperref}

\usepackage{subcaption}
\usepackage{bm}
\usepackage{float}

\newtheorem{theorem}{Theorem}
\newtheorem{lemma}{Lemma}
\newtheorem{corollary}{Corollary}
\theoremstyle{definition}
\newtheorem{remark}{Remark}
\usepackage[backend=biber]{biblatex}
\providecommand{\keywords}[1]
{
  \textbf{\textit{Key words: }} #1
}

\providecommand{\subjclass}[1]
{
  \textbf{\textit{MSC2020: }} #1
}

\begin{document}



\title{Dynamic Regret for Online Regression in RKHS via Discounted VAW and Subspace Approximation}

\author{
  Dmitry B. Rokhlin\thanks{The research was supported by the Regional Mathematical Center of the Southern Federal University with the Agreement no. 075-02-2026-1316 of the Ministry of Science and Higher Education of Russia.} \\
  Institute of Mathematics, Mechanics and Computer Sciences of the\\
  Southern Federal University \\
  Regional Scientific and Educational Mathematical Center of the\\
  Southern Federal University \\
  \texttt{dbrohlin@sfedu.ru}
  \and
   Georgiy A. Karapetyants\\
  Institute of Mathematics, Mechanics and Computer Sciences of the\\
  Southern Federal University \\
  \texttt{karape@sfedu.ru}}

 \date{} 

\maketitle

\begin{abstract}
We study online regression with the square loss in a reproducing kernel Hilbert
space under a dynamic regret criterion. The learner is compared with a
time-varying comparator sequence, and the bounds depend on its path length in
the RKHS norm. The proposed method transfers the finite-dimensional discounted Vovk--Azoury--Warmuth approach of Jacobsen \& Cutkosky (2024) to the RKHS setting by means of finite-dimensional subspace approximations. For a fixed subspace, we run a VAW-based ensemble of discounted VAW forecasters over a geometric grid of discount factors. The additional approximation error is controlled by the uniform projection error of
kernel sections. 

We then introduce a general orthogonal truncation method: starting from a
feature expansion of the kernel, we construct the associated RKHS by introducing
an inner product that makes the feature functions orthonormal, and then use the
spans of the first basis functions as finite-dimensional approximation spaces.
The resulting subspace reduction is applied to several approximation schemes.
Explicit feature expansions yield fast-regime bounds for Gaussian
and analytic dot-product kernels. Mercer truncations provide a spectral
approximation method and lead to dynamic regret bounds in fast and slow regimes, depending on the eigenvalue decay. Finally, we study subspaces spanned by kernel sections and
apply this construction to Mat\'ern kernels.
\end{abstract}



\section{Introduction}
Online linear regression with the square loss is a classical problem in online learning. In the finite-dimensional setting, one compares the learner to the best fixed linear predictor chosen in hindsight and measures performance by the corresponding regret. Applying the general Aggregating Algorithm (AA) to this problem, Vovk derived a forecaster with static regret bounds of order $O(d\ln T)$, where $d$ is the space dimension and $T$ is the number of time steps \cite{Vovk2001competitive}. These results were rederived in \cite{Azoury2001relative} using the Follow-the-Regularized-Leader (FTRL) approach.
Following the standard terminology, we refer to this forecaster as the Vovk--Azoury--Warmuth (VAW) algorithm.

A natural nonparametric extension is to compete with predictors from a reproducing kernel Hilbert space (RKHS). To address this problem, the Kernel Aggregating Algorithm Regression (KAAR) was introduced, and a corresponding regret bound with a log-determinant term involving the empirical kernel matrix was established in \cite{Gammerman2004on-line}. As observed in \cite{Vovk2006on-line}, this bound immediately implies a regret bound of order $O(\|f\|_{\mathcal H}\sqrt{T})$, provided that the horizon $T$ and an upper bound on the RKHS norm of the comparator $f$ are known in advance. One approach used in \cite{Vovk2006on-line} to remove the latter assumption is based on a second-level AA aggregation of the KAAR predictors over all regularization parameters. The regret rates mentioned above are optimal in $T$ in both finite-dimensional and RKHS settings.

More recently, \cite{Zadorozhnyi2021Online} analyzed the exact KAAR on Sobolev spaces. In particular, they obtained essentially optimal static regret rates
$
O\!\left(T^{\frac{d}{2\beta+d}+\varepsilon}\ln T\right),
$
using upper bounds on the effective dimension for Sobolev RKHS spaces $H^\beta$ with $\beta>d/2$. The optimality conclusion relies on the results of \cite{Rakhlin2014online} and \cite{Gaillard2015a_chaining}, concerning the case where an upper bound on the comparator norm is known in advance. In \cite{Jezequel2019efficient}, KAAR was revisited from a computational viewpoint: the authors replaced the RKHS in the basic update rule by carefully chosen finite-dimensional subspaces. In particular, for Gaussian kernels they obtained polylogarithmic regret bounds of order $O((\ln T)^{d+1})$.

The dynamic counterpart of static regret compares the learner not with a single predictor, but with a time-varying comparator sequence. In the general convex setting, sublinear dynamic regret is possible only under regularity assumptions on this sequence, typically expressed through its path length $P_T$. In the finite-dimensional least-squares setting, the most relevant recent progress is due to \cite{Jacobsen2024online}, where a discounted VAW (DVAW) forecaster was introduced, yielding, in simplified form, a dynamic-regret bound of order
$
\widetilde O\!\left(1\vee\sqrt{P_T\,T}\right).
$ 
Note that under the stronger assumption that the comparator sequence is uniformly bounded and that this bound is known in advance, a similar path-length-type guarantee for exp-concave losses follows from the results of \cite{Yuan2020trading}, whose analysis relies on a discounted online Newton-type method. In the same bounded-comparator setting, \cite{Baby2021optimal} obtained a sharper guarantee for exp-concave losses via a strongly adaptive improper-learning strategy.

In the RKHS setting, dynamic regret under unconstrained square loss appears to
be much less developed. The closest directly comparable result is
\cite{Rokhlin2025hierarchical}, where the finite-dimensional DVAW approach of \cite{Jacobsen2024online} is
combined with random feature approximations and VAW aggregation over both the
discount parameter and the number of features. This yields an expected
dynamic-regret bound of order $O\!\left(T^{2/3}P_T^{1/3}+\sqrt T\ln T\right),$
where \(P_T\) is the RKHS path length of the comparator sequence. 
We are not aware of other direct dynamic regret results in RKHS under unconstrained square loss. 
In a bounded RKHS setting, \cite{Baby2021dynamic} proved a dynamic-regret guarantee for Gaussian
kernels by combining strongly adaptive aggregation with the PKAWV base learners
of \cite{Jezequel2019efficient}; their result assumes bounded labels and a comparator sequence contained in a known RKHS ball.

The present paper develops a deterministic subspace counterpart of the
random-feature RKHS-DVAW approach. The finite-dimensional DVAW and VE-DVAW
results of \cite{Jacobsen2024online,Rokhlin2025ensembling} are used as
online-learning components. The new contribution is a reduction of the RKHS dynamic-regret problem to the finite-dimensional case via deterministic subspace approximation.
This reduction separates the finite-dimensional online-learning part from the deterministic kernel approximation part and leads to kernel-dependent deterministic guarantees in
the fast and slow approximation regimes analyzed below.

This subspace viewpoint is related to \cite{Jezequel2019efficient}, but the
role of the approximation is different. 
There the subspaces are used to approximate the
KAAR update in a static-regret setting, and the error is controlled through
data-dependent covariance operators. In the present paper the subspaces are used to compare an
arbitrary time-varying RKHS comparator with its projections, and the error is
controlled by uniform approximation properties of the kernel sections.

The paper is organized as follows. In Section~2, we recall the finite-dimensional discounted VAW forecaster of \cite{Jacobsen2024online}, as well as the VE-DVAW ensemble of \cite{Rokhlin2025ensembling} built on a geometric grid of discount factors, together with the corresponding dynamic-regret bound.

In Section~3, we pass to the RKHS setting and derive a general reduction theorem for VE-DVAW run on an arbitrary finite-dimensional subspace $V_m\subset\mathcal H$ (Theorem \ref{thm:VE_DVAW_arbitrary_subspace}). This yields a dynamic-regret bound showing the trade-off between the complexity of the finite-dimensional problem and the approximation error. At the abstract level, we introduce fast and slow approximation regimes and deduce the corresponding regret bounds (Theorems \ref{thm:fast_regime_general} and \ref{thm:adaptive_over_m_slow_general}). In the fast regime, the subspace dimension can be fixed in advance from the decay rate of the approximation error. In the slow regime, the optimal dimension is unknown and must be learned adaptively, which leads to an additional top-level VAW aggregation over a dyadic grid of subspace dimensions. 
Section~3 also contains a general orthogonal truncation scheme. It shows that a pointwise convergent expansion of the form $k(x,y)=\sum_{j=1}^\infty g_j(x)g_j(y)$
with linearly independent continuous functions $g_j$ generates an RKHS, in which $(g_j)$ forms an orthonormal basis (Lemma \ref{lem:explicit_feature_construction}). This provides a unified way to construct finite-dimensional approximation spaces $V_m$ by truncation and will be used repeatedly in the subsequent sections.

In Section~4, the general orthogonal truncation scheme is applied successively to polynomial, Gaussian, and analytic dot-product kernels. The polynomial case reduces to the finite-dimensional setting, since the corresponding RKHS is itself finite-dimensional.
In the Gaussian and analytic dot-product cases, the truncation error decays exponentially, so these kernels fall into the fast approximation regime.

In Section~5, we study Mercer-type truncations. The key step is to estimate the uniform remainder of the truncated Mercer series. For this purpose, we use the results of \cite{Takhanov2023speed}, which provide bounds in terms of the tails of the eigenvalue sequence. This leads to fast-regime bounds under stretched-exponential eigenvalue decay and to slow-regime bounds under polynomial decay. The resulting theory is then applied to tensor-product analytic kernels on the hypercube and to smooth Mat\'ern kernels. In particular, for Gaussian kernels this approach recovers the correct polylogarithmic order of regret. For Mat\'ern kernels, however, the resulting rate is not optimal, and the next section shows that a different approximation scheme yields sharper bounds.

In Section~6, we study an approximation scheme based on subspaces spanned by kernel sections. Given a finite set of points $Z$, we consider the span of the corresponding sections $k(\cdot,z)$, $z\in Z$, and control the approximation error through covering numbers in the pseudometric induced by the kernel. We then apply this construction to Mat\'ern kernels. In the rough regime, corresponding to the smoothness parameter $\nu<1$, the small-scale behavior of the pseudometric leads directly to the relevant slow-regime bounds. In the smooth regime, $\nu>1$, we instead use the Sobolev characterization of the Mat\'ern RKHS together with sharper estimates for kernel section spaces, following the ideas of \cite{Zadorozhnyi2021Online}. This leads to an improved approximation rate in the smooth Mat\'ern case, and yields sharper bounds than those obtained from the Mercer-type analysis of Section~5. The borderline case \(\nu=1\) is treated separately and leads to
additional logarithmic factors.

\section{Discounted VAW forecaster in the finite-dimensional case}
We first recall the finite-dimensional online linear regression problem with
the square loss.
At each round $t=1,\dots,T$ the environment reveals a feature vector $z_t\in\mathbb{R}^m$.
In addition, following
\cite{Jacobsen2024online}, we allow the learner to receive a hint
\(\tilde y_t\in\mathbb R\) before making its prediction. No assumption is made
on how \(\tilde y_t\) is generated: it may come from an external source or
be defined by the learner itself.
Intuitively, \(\tilde y_t\) is an auxiliary prediction of the current outcome
available before the learner acts. For example, it may be the prediction of a
standard non-discounted VAW forecaster; in the VE-DVAW ensemble below, the hint
itself is also included as the special expert \(\mathcal A_0\).
Then the learner outputs a weight vector $w_t\in\mathbb R^m$ and predicts
$
\hat y_t=\langle w_t,z_t\rangle. $
After that, the outcome $y_t\in\mathbb R$ is revealed and the learner incurs the loss $\ell_t(w_t)$, where
\[
\ell_t(w)
=
\frac12\bigl(y_t-\langle w,z_t\rangle\bigr)^2.
\]
To measure the performance of the strategy $w_{1:T}:=(w_1,\dots,w_T)$ in a dynamic environment,
we consider the dynamic regret
\[
R_T(u_{1:T})
=
\sum_{t=1}^T \ell_t(w_t)\;-\;\sum_{t=1}^T \ell_t(u_t),
\]
where $u_{1:T}=(u_1,\dots,u_T)$ is an arbitrary comparator sequence with $u_t\in\mathbb{R}^m$.

Fix $\lambda>0, \gamma\in(0,1]$, and put
\[
F_t^\gamma(w)
=
\gamma^t \frac{\lambda}{2}\|w\|_2^2 \;+\; \sum_{s=1}^t \gamma^{t-s}\,\ell_s(w),
\qquad t\ge 0.
\]
The discounted VAW (DVAW) forecaster of \cite{Jacobsen2024online} can be written in the  FTRL form:
\begin{equation}
\label{eq:DVAW_update}
w_t
\;\in\;
\arg\min_{w\in\mathbb{R}^m}
\left\{
\frac12\bigl(\tilde y_t-\langle z_t,w\rangle\bigr)^2
+\gamma F_{t-1}^\gamma(w)
\right\},
\qquad t=1,\dots,T.
\end{equation}
For later use, recall that the standard VAW forecaster is obtained from
\eqref{eq:DVAW_update} by taking \(\gamma=1\) and \(\tilde y_t=0\).
More generally, for feature vectors \(q_t\in\mathbb R^p\), it is given by
\begin{equation}
\label{eq:standard_VAW}
\alpha_t
=
\arg\min_{\alpha\in\mathbb R^p}
\left\{
\frac{\lambda}{2}\|\alpha\|_2^2
+
\frac12\sum_{s=1}^{t-1}
\bigl(y_s-\langle\alpha,q_s\rangle\bigr)^2
+
\frac12\langle\alpha,q_t\rangle^2
\right\},
\qquad
\hat y_t=\langle\alpha_t,q_t\rangle .
\end{equation}

\begin{theorem}[\cite{Jacobsen2024online}, Theorem 3.1]
\label{thm:JC31}
Let $\lambda>0$ and $\gamma\in(0,1]$. Then for any comparator sequence
$u_{1:T}$ the DVAW forecaster algorithm, which we denote by $\mathcal A_\gamma(\lambda)$, satisfies
\begin{align*}
 R_T^{\mathcal A_\gamma(\lambda)}(u_{1:T})
&\le
\gamma\frac{\lambda}{2}\|u_1\|_2^2
+\frac{m}{2}\max_{1\le t\le T}\Delta_t^2
\ln\!\left(
1+\frac{\sum_{t=1}^T \gamma^{T-t}\|z_t\|_2^2}{\lambda m}
\right)\\
&+\gamma\sum_{t=1}^{T-1}\bigl[F_t^\gamma(u_{t+1})-F_t^\gamma(u_t)\bigr]_+
+\frac{m}{2}\ln(1/\gamma)\,\Delta_{1:T}^2, 
\end{align*}
where $\Delta_t^2 = (y_t-\tilde y_t)^2$ and $\Delta_{1:T}^2 = \sum_{t=1}^T (y_t-\tilde y_t)^2$.
\end{theorem}

Denote by 
\[
P_T(u_{1:T})=\sum_{t=1}^{T-1}\|u_{t+1}-u_t\|_2
\]
the path length of the comparator sequence $u_{1:T}$. The bound of Theorem \ref{thm:JC31} was simplified in \cite{Rokhlin2025ensembling} to a more explicit form. 

\begin{lemma}[\cite{Rokhlin2025ensembling}, Lemma~3.1]
\label{lem:RK_path_length}
Assume that $\|z_t\|_2\le a$, $|y_t|\le Y$, $\|u_t\|_2\le R$. For $\gamma\in (0,1]$ put $\eta=\gamma/(1-\gamma)$.
Then the DVAW forecaster $\mathcal A_\gamma(\lambda)$ satisfies
\[
R_T^{\mathcal A_\gamma(\lambda)}(u_{1:T})
\le
\eta\,a(aR+Y)\,P_T(u_{1:T})
+\frac{m}{2\eta}\,\Delta_{1:T}^2
+\lambda R\,P_T(u_{1:T})
+\frac{m}{2}\max_{1\le t\le T}\Delta_t^2\,
\ln\!\left(1+\frac{a^2T}{\lambda m}\right)
+\frac{\lambda}{2}R^2,
\]
where, for $\gamma=1$, the right-hand side is understood in the extended
sense, with $0\cdot\infty=0$ in the static case $P_T(u_{1:T})=0$.
\end{lemma}

Minimizing the bound of Lemma~\ref{lem:RK_path_length} over $\eta\in [0,\infty)$ we get
\[
\eta^\star
=
\sqrt{\frac{m\,\Delta_{1:T}^2}{2a(aR+Y)\,P_T(u_{1:T})}},
\qquad
\gamma^\star
=
\frac{\eta^\star}{1+\eta^\star}.
\]
In the static case \(P_T(u_{1:T})=0\), the formal optimum corresponds to
\(\eta=\infty\), that is, to \(\gamma=1\).
The related optimal bound is
\[
R_T^{\mathcal A_{\gamma^\star}(\lambda)}(u_{1:T})
\;\le\;
\sqrt{2m\,a(aR+Y)\,\Delta_{1:T}^2\,P_T(u_{1:T})}
+\lambda R\,P_T(u_{1:T})
+\frac{m}{2}\max_{1\le t\le T}\Delta_t^2\,
\ln\!\Bigl(1+\frac{a^2T}{\lambda m}\Bigr)
+\frac{\lambda}{2}R^2.
\]

However, $\eta^\star$ depends on the unknown quantities $\Delta_{1:T}^2$ and $P_T(u)$, and thus cannot be used directly. To overcome this difficulty, following \cite{Jacobsen2024online} and \cite{Rokhlin2025ensembling}, we consider an ensemble of DVAW forecasters $\mathcal A_{\gamma_k}(\lambda)$, where $\gamma_k$ belong to the grid
\begin{equation} \label{eq:grid}
\eta_{\min}=2m,
\qquad
\eta_{\max}=mT,
\qquad
\eta_i=\min(\eta_{\min}b^i,\eta_{\max}),\quad i\in\mathbb{Z}_+,
\qquad
\gamma_i=\frac{\eta_i}{1+\eta_i},
\end{equation}
with $b>1$. Let
\[
S_\gamma=\{0\}\cup\{\gamma_i:\ i\in\mathbb{Z}_+\},
\qquad
N=|S_\gamma|=O(\ln T).
\]
$\mathcal A_0(\lambda)$ is defined as the hint-based expert $\hat y_{t,0}:=\tilde y_t$. 
We collect predictions into the meta-feature vector
\[
\mathbf z_t
=
(\hat y_{t,0},\hat y_{t,1},\dots,\hat y_{t,N-1})
\in\mathbb R^N,
\]
and apply the standard VAW forecaster $\mathcal A(\lambda)$ defined by \eqref{eq:standard_VAW} with
\(p=N\) and \(q_t=\mathbf z_t\). Following \cite{Rokhlin2025ensembling},
we call the resulting ensemble VE-DVAW (VAW-ensembled DVAW).

\begin{theorem}[\cite{Rokhlin2025ensembling}, Theorem 3.2]
\label{thm:VE_DVAW_m}
Assume that $\|z_t\|_2\le a$, $|y_t|\le Y$, $|\tilde y_t|\le \tilde Y$, $\|u_t\|_2\le R$.
The VE-DVAW ensemble over the grid (\ref{eq:grid}) satisfies 
\begin{align*}
R_T^{\mathcal A(\lambda)}(u_{1:T})
& \le 
\frac{\lambda}{2}
+\frac{N Y^2}{2}\ln\!\left(1+\frac{Z_T^2}{\lambda N}\right)
+
(1+b)\sqrt{\frac{m}{2}\,a(aR+Y)\,P_T(u_{1:T})\,\Delta_{1:T}^2} \nonumber\\
& +
\lambda R\,P_T(u_{1:T})+\frac{\lambda}{2}R^2
+
\frac{m}{2}\max_{t\le T}\Delta_t^2 \ln\!\left(1+\frac{a^2T}{\lambda m}\right),
\end{align*}
where $Z_T^2=\sum_{t=1}^T \|\mathbf z_t\|_2^2$ admits the bound
\[
Z_T^2
\le
T\tilde Y^2
+
(N-1)\left[
4TY^2
+
T\,(Y+\tilde Y)^2
+
2m\,(Y+\tilde Y)^2\ln\!\left(1+\frac{a^2T}{\lambda m}\right)
\right].
\]
\end{theorem}

\section{A general RKHS reduction via finite-dimensional subspaces}
\label{sec:abstract_truncation_scheme}
In what follows we denote by $X\subset \mathbb R^d$ a compact set with nonempty interior. Recall that an RKHS $\mathcal H$ on $X$ is a Hilbert space of functions $f:X\to\mathbb R$ such that the evaluation functional $f\mapsto f(x)$ is continuous for every $x\in X$ \cite{Steinwart2008support}. By the Riesz representation theorem, for every $x\in X$ there exists a unique function $k(\cdot,x)\in\mathcal H$ such that $f(x)=\langle f,k(\cdot,x)\rangle_{\mathcal H}$ for all $f\in\mathcal H$. The function $k:X\times X\to\mathbb R$ is called the reproducing kernel of $\mathcal H$.  Put
\[
\kappa=\sup_{x\in X}\sqrt{k(x,x)}.
\]

Let $V\subset\mathcal H$ be a finite-dimensional subspace. Denote by $\Pi_V:\mathcal H\to V$  the orthogonal projector. Following \cite{Santin2016approximation}, define the generalized power function by
\[
P_V(x)
=
\sup_{\|f\|_{\mathcal H}\le 1}|f(x)-(\Pi_V f)(x)|=\sup_{\|f\|_{\mathcal H}\le 1}
|\langle f,(I-\Pi_V)k(\cdot,x)\rangle_{\mathcal H}|=\|(I-\Pi_V)k(\cdot,x)\|_{\mathcal H},
\qquad x\in X,
\]
where we used the reproducing property and self-adjointness of $\Pi_V$. Put
$
\mathcal E(V)
=
\sup_{x\in X}P_V(x).
$
Clearly, for any $f\in\mathcal H$ we have
\begin{equation}
\label{eq:basic_projection_error_bound}
|f(x)-(\Pi_V f)(x)|
\le
\|f\|_{\mathcal H}\,\mathcal E(V).
\end{equation}

The finite-dimensional results of Section~2 will be used as a black-box
prediction procedure after passing from the original input points \(x_t\in X\)
to finite-dimensional feature vectors. Such a feature map
is obtained by choosing functions \(g_1,\dots,g_m\in\mathcal H\) and setting
$
\Psi_m(x)=(g_1(x),\dots,g_m(x)).
$
In what follows we take \(g_1,\dots,g_m\) to be an orthonormal system in
\(\mathcal H\), and denote
$
V_m=\operatorname{span}\{g_1,\dots,g_m\}.
$
Then the abstract feature vectors \(z_t\) from Section~2 are instantiated as
$
z_t=\Psi_m(x_t).
$

\subsection{VE-DVAW over a finite-dimensional subspace}
\begin{lemma}
\label{lem:approx_sq_loss_via_EV}
Assume that $|y_t|\le Y$ and $f_{1:T}\subset\mathcal H$ satisfy
$ \|f_t\|_{\mathcal H}\le R_f$, $t=1,\dots,T.$ Put $f_{t,V}=\Pi_V f_t.$ Then
\[
\sum_{t=1}^T
\Bigl(
\ell_t(f_{t,V})-\ell_t(f_t)
\Bigr)
\le
T\left(
R_f(Y+\kappa R_f)\mathcal E(V)
+\frac12 R_f^2 \mathcal E(V)^2
\right).
\]
\end{lemma}

\begin{proof}
Put $\delta_t:=f_t(x_t)-f_{t,V}(x_t).$ Then
\[
\ell_t(f_{t,V})-\ell_t(f_t)
=
\frac12\bigl(y_t-f_t(x_t)+\delta_t\bigr)^2
-
\frac12\bigl(y_t-f_t(x_t)\bigr)^2
=
\delta_t\bigl(y_t-f_t(x_t)\bigr)+\frac12\delta_t^2.
\]
By the inequality \eqref{eq:basic_projection_error_bound},
\[
|\delta_t|
=
|f_t(x_t)-f_{t,V}(x_t)|
\le
\|f_t\|_{\mathcal H}\,\mathcal E(V)
\le
R_f\,\mathcal E(V).
\]
Also, by the reproducing property,
\[
|f_t(x_t)|
\le
\|f_t\|_{\mathcal H}\,\|k(\cdot,x_t)\|_{\mathcal H}
=
\|f_t\|_{\mathcal H}\sqrt{k(x_t,x_t)}
\le
R_f\kappa.
\]
Hence
\[
|y_t-f_t(x_t)|
\le
|y_t|+|f_t(x_t)|
\le
Y+\kappa R_f.
\]
Finally, combining the above bounds, we get
\[
\ell_t(f_{t,V})-\ell_t(f_t)
\le
|\delta_t|\,|y_t-f_t(x_t)|+\frac12\delta_t^2
\le
R_f(Y+\kappa R_f)\mathcal E(V)
+\frac12 R_f^2\mathcal E(V)^2.
\]
Summing over \(t=1,\dots,T\) completes the proof.
\end{proof}

Let \(V_m\subset\mathcal H\) be an \(m\)-dimensional subspace, and let
\(g_1,\dots,g_m\) be an orthonormal basis of \(V_m\) in \(\mathcal H\): $
\langle g_i,g_j\rangle_{\mathcal H}=\delta_{ij}$, $1\le i,j\le m$.
Define
\begin{equation}
\label{eq:feature_map_subspace}
\Psi_{V_m}(x)=(g_1(x),\dots,g_m(x))\in\mathbb R^m.
\end{equation}
Then every \(f\in V_m\) can be written uniquely as
\[
f=\sum_{j=1}^m u_j g_j
\qquad\text{with some }u=(u_1,\dots,u_m)\in\mathbb R^m,
\]
and therefore
$
f(x)=\langle u,\Psi_{V_m}(x)\rangle$,
$x\in X.
$
Moreover, 
\begin{equation}
\label{eq:subspace_coordinate_isometry}
\|f\|_{\mathcal H}=\|u\|_2,
\end{equation}
\begin{equation}
\label{eq:subspace_feature_norm_bound}
\|\Psi_{V_m}(x)\|_2^2=\sum_{i=1}^m g_i(x)^2=\sum_{i=1}^m \langle g_i,k(\cdot,x)\rangle_{\mathcal H}^2
=
\|\Pi_{V_m}k(\cdot,x)\|_{\mathcal H}^2
\le
\|k(\cdot,x)\|_{\mathcal H}^2
=
k(x,x)
\le
\kappa^2.
\end{equation}

For a function \(f\in\mathcal H\) we use the shorthand notation
$
\ell_t(f):=\frac12\bigl(y_t-f(x_t)\bigr)^2.
$
Accordingly, for a comparator sequence $f_{1:T}=(f_1,\dots,f_T)$, define the dynamic regret
\[
R_T(f_{1:T})
=
\sum_{t=1}^T
\bigl(\ell_t(\hat y_t)-\ell_t(f_t)\bigr)
=
\sum_{t=1}^T
\left(
\frac12 (y_t-\hat y_t)^2
-
\frac12 (y_t-f_t(x_t))^2
\right).
\]
Let 
$
P_T^{\mathcal H}(f_{1:T})
=
\sum_{t=1}^{T-1}\|f_{t+1}-f_t\|_{\mathcal H}
$
be the path length.

\begin{theorem}
\label{thm:VE_DVAW_arbitrary_subspace}
Assume that $|y_t|\le Y$, $|\tilde y_t|\le \tilde Y$, and let
$f_{1:T}\subset\mathcal H$ satisfy $\|f_t\|_{\mathcal H}\le R_f$, $t=1,\dots,T.$
Run the VE-DVAW ensemble \(\mathcal A(\lambda)\) on the features
\[
z_t=\Psi_{V_m}(x_t)\in\mathbb R^m,
\]
with regularization parameter \(\lambda>0\), and with the discount grid \eqref{eq:grid}.
Then
\begin{align}
R_T^{\mathcal A(\lambda)}(f_{1:T})
&=
\sum_{t=1}^T
\Bigl(
\ell_t(\hat y_t)-\ell_t(f_t)
\Bigr)
\nonumber\\
&\le
\frac{\lambda}{2}
+\frac{N Y^2}{2}\ln\!\left(1+\frac{Z_{T,m}^2}{\lambda N}\right)
+(1+b)\sqrt{
\frac{m}{2}\,
\kappa(\kappa R_f+Y)\,
P_T^{\mathcal H}(f_{1:T})\,
\Delta_{1:T}^2
}
\nonumber\\
&
+\lambda R_f\,P_T^{\mathcal H}(f_{1:T})
+\frac{\lambda}{2}R_f^2
+\frac{m}{2}\max_{1\le t\le T}\Delta_t^2
\ln\!\left(1+\frac{\kappa^2T}{\lambda m}\right)
\nonumber\\
&+
T\left(
R_f(Y+\kappa R_f)\mathcal E(V_m)
+\frac12R_f^2\mathcal E(V_m)^2
\right),
\label{eq:general_subspace_bound}
\end{align}
where 
\(Z_{T,m}^2\) denotes the quantity \(Z_T^2\) from
Theorem~\ref{thm:VE_DVAW_m} corresponding to the present choice
\(z_t=\Psi_{V_m}(x_t)\), and admits the same bound as in
Theorem~\ref{thm:VE_DVAW_m} with \(a=\kappa\).
\end{theorem}

\begin{proof} Let \(u_t\in\mathbb R^m\) denote the coordinate vector of \(f_{t,m}=\Pi_{V_m}f_t\) in the
orthonormal basis \(g_1,\dots,g_m\). Then, by \eqref{eq:subspace_coordinate_isometry},
\[
\|u_t\|_2=\|f_{t,m}\|_{\mathcal H}\le \|f_t\|_{\mathcal H}\le R_f.
\]
Also, since orthogonal projection is non-expansive,
\[
\sum_{t=1}^{T-1}\|u_{t+1}-u_t\|_2
=
\sum_{t=1}^{T-1}\|f_{t+1,m}-f_{t,m}\|_{\mathcal H}
\le
P_T^{\mathcal H}(f_{1:T}).
\]
For $z_t=\Psi_{V_m}(x_t)$, we have, by construction,
$
f_{t,m}(x_t)=\langle u_t,z_t\rangle,
$
and by \eqref{eq:subspace_feature_norm_bound},
$
\|z_t\|_2\le \kappa.
$
Applying Theorem~\ref{thm:VE_DVAW_m} to the finite-dimensional sequence
\((z_t,y_t)_{t=1}^T\) and the comparator sequence \(u_{1:T}\), we get
\begin{align*}
\sum_{t=1}^T
\Bigl(
\ell_t(\hat y_t)-\ell_t(f_{t,m})
\Bigr)
&=
\sum_{t=1}^T
\left(
\frac12 (y_t-\hat y_t)^2
-
\frac12 (y_t-\langle u_t,z_t\rangle)^2
\right)
\\
&\le
\frac{\lambda}{2}
+\frac{N Y^2}{2}\ln\!\left(1+\frac{Z_{T,m}^2}{\lambda N}\right)
+(1+b)\sqrt{
\frac{m}{2}\,
\kappa(\kappa R_f+Y)\,
P_T^{\mathcal H}(f_{1:T})\,
\Delta_{1:T}^2
}
\\
&\quad
+\lambda R_f\,P_T^{\mathcal H}(f_{1:T})
+\frac{\lambda}{2}R_f^2
+\frac{m}{2}\max_{1\le t\le T}\Delta_t^2
\ln\!\left(1+\frac{\kappa^2T}{\lambda m}\right).
\end{align*}
Now decompose
\[
R_T^{\mathcal A(\lambda)}(f_{1:T})
=
\sum_{t=1}^T
\Bigl(
\ell_t(\hat y_t)-\ell_t(f_{t,m})
\Bigr)
+
\sum_{t=1}^T
\Bigl(
\ell_t(f_{t,m})-\ell_t(f_t)
\Bigr),
\]
and apply Lemma~\ref{lem:approx_sq_loss_via_EV} with \(V=V_m\).
\end{proof}

In the next Corollary we formulate a rough regret bound, which is a direct consequence of 
\eqref{eq:general_subspace_bound}.

\begin{corollary}
\label{cor:rough_bound_subspace}
Assume the conditions of Theorem~\ref{thm:VE_DVAW_arbitrary_subspace}. Suppose that
\(b,\kappa,Y,\tilde Y,R_f\) are fixed constants. Let $\lambda>0$ be fixed, and assume
that \(m=O(T)\). Then
\[
R_T^{\mathcal A(\lambda)}(f_{1:T})
=
O\left(
(\ln T)^2
+
\sqrt{m\,T\,P_T^{\mathcal H}(f_{1:T})}
+
m\ln\!\left(1+\frac{T}{m}\right)
+
T\,\mathcal E(V_m)\right).
\]
\end{corollary}

\begin{proof}
Note that $P_T^\mathcal H(f_{1:T})=O(T)$. Hence $P_T^\mathcal H(f_{1:T})=O\left(\sqrt{T P_T^\mathcal H(f_{1:T})}\right)$.
Theorem~\ref{thm:VE_DVAW_arbitrary_subspace} yields
\begin{equation}
\label{eq:rough_bound_subspace_mid}
R_T^{\mathcal A(\lambda)}(f_{1:T})
=
O\left(
\frac{N Y^2}{2}\ln\!\left(1+\frac{Z_{T,m}^2}{\lambda N}\right)
+
\sqrt{mT\,P_T^{\mathcal H}(f_{1:T})}
+
m\ln\!\left(1+\frac{T}{m}\right)
+
T\mathcal E(V_m)\right),
\end{equation}
where we used that \(\Delta_{1:T}^2=O(T)\), \(\max_{1\le t\le T}\Delta_t^2=O(1)\), and
$
T\mathcal E(V_m)^2=O(T\mathcal E(V_m)),
$
since \(\mathcal E(V_m)\le \kappa\).
Furthermore,
\[
Z_{T,m}^2
=
O\!\left(
NT+Nm\ln\!\left(1+\frac{T}{m}\right)
\right).
\]
Since \(N=O(\ln T)\) and \(m=O(T)\), we get
\[
\frac{N Y^2}{2}\ln\!\left(1+\frac{Z_{T,m}^2}{\lambda N}\right)
=
O\!\left((\ln T)^2\right).
\]
Substituting this bound into \eqref{eq:rough_bound_subspace_mid} gives the claim.
\end{proof}

\subsection{Fast and slow approximation regimes}
The next two results describe the behavior of the general subspace scheme in the
fast and slow approximation regimes, respectively.

\begin{theorem}[Fast approximation regime]
\label{thm:fast_regime_general}
Assume the conditions of Corollary~\ref{cor:rough_bound_subspace}. Suppose that a family of
\(m\)-dimensional subspaces \((V_m)_{m\ge1}\) is given such that
\[
\mathcal E(V_m)\le C_1\exp(-C_2 m^\alpha),
\qquad \alpha>0.
\]
Assume that \(C_2\) is known, and choose
\[
m(T)=\left\lceil \left(\frac{\ln T}{C_2}\right)^{1/\alpha}\right\rceil.
\]
Then the VE-DVAW forecaster run on the feature map
\(\Psi_{V_{m(T)}}\), defined in \eqref{eq:feature_map_subspace}, satisfies
\[
R_T^{\mathcal A(\lambda)}(f_{1:T})
=
O\!\left(
\sqrt{T\,P_T^{\mathcal H}(f_{1:T})}\,
(\ln T)^{\frac{1}{2\alpha}}
+
(\ln T)^{1+\frac{1}{\alpha}}
+
(\ln T)^2
\right).
\]
\end{theorem}

\begin{proof} Substituting \(m=m(T)\) into Corollary~\ref{cor:rough_bound_subspace}, we obtain
the stated bound, since
\[
T\,\mathcal E(V_{m(T)})
\le
C_1 T \exp(-C_2 m(T)^\alpha)
\le
C_1. 
\]
\end{proof}

\begin{theorem}[Slow approximation regime]
\label{thm:adaptive_over_m_slow_general}
Assume the conditions of Corollary~\ref{cor:rough_bound_subspace}. Suppose that a family of
\(m\)-dimensional subspaces \((V_m)_{m\ge1}\) is given such that
\[
\mathcal E(V_m)\le C m^{-\beta},
\qquad \beta>0.
\]
Let
\[
    \mathcal M_{T,\beta}
    :=
    \left\{
        2^j:
        0\le j\le
        \left\lfloor
            \frac{\log_2 T}{\beta+1}
        \right\rfloor
    \right\}.
\]
For each \(m\in\mathcal M_{T,\beta}\), denote by
\(\mathcal A_m\) the VE-DVAW forecaster described in
Theorem~\ref{thm:VE_DVAW_arbitrary_subspace}, run on the
feature map \(\Psi_{V_m}\). Let \(\hat y_t^{(m)}\) be its
prediction at time \(t\).

Let \(\overline{\mathcal A}\) be the top-level standard VAW
forecaster \eqref{eq:standard_VAW}, with
$ p=|\mathcal M_{T,\beta}|$, $q_t=
    \bigl(\hat y_t^{(m)}\bigr)_{m\in\mathcal M_{T,\beta}}$,
and regularization parameter \(\lambda=1\).
Then, for every comparator sequence \(f_{1:T}\subset\mathcal H\) satisfying
$
\|f_t\|_{\mathcal H}\le R_f$,
$t=1,\dots,T$,
we have
\[
R_T^{\overline{\mathcal A}}(f_{1:T})=
O\left(
T^{\frac{\beta+1}{2\beta+1}}
\Bigl(P_T^{\mathcal H}(f_{1:T})\Bigr)^{\frac{\beta}{2\beta+1}}
+
T^{\frac{1}{\beta+1}}(\ln T)^{\frac{\beta}{\beta+1}}\right).
\]
\end{theorem}
\begin{proof}
For every \(m\in\mathcal M_{T,\beta}\), decompose the regret
of the top-level forecaster as
\begin{align*}
R_T^{\overline{\mathcal A}}(f_{1:T})
&=
R_T^{\overline{\mathcal A}}(e_m)
+
R_T^{\mathcal A_m}(f_{1:T}), \qquad
e_m
:=
\bigl(
    \mathbf 1_{\{m=r\}}
\bigr)_{r\in\mathcal M_{T,\beta}}
\in
\mathbb R^{|\mathcal M_{T,\beta}|},
\\
R_T^{\overline{\mathcal A}}(e_m)
&=
\sum_{t=1}^T
\left(
\frac12 (y_t-\hat y_t)^2
-
\frac12 (y_t-\hat y_t^{(m)})^2
\right),
\\
R_T^{\mathcal A_m}(f_{1:T})
&=
\sum_{t=1}^T
\left(
\frac12 (y_t-\hat y_t^{(m)})^2
-
\frac12 (y_t-f_t(x_t))^2
\right).
\end{align*}

Put \(P:=P_T^{\mathcal H}(f_{1:T})\).
Since \(m\le T\) for all \(m\in\mathcal M_{T,\beta}\),
Corollary~\ref{cor:rough_bound_subspace} gives
\begin{equation}
\label{eq:adaptive_m_general_fixed_m2}
R_T^{\mathcal A_m}(f_{1:T})
=
O\left(
(\ln T)^2
+
\sqrt{mTP}
+
m\ln\!\left(1+\frac{T}{m}\right)
+
Tm^{-\beta}
\right),
\qquad
m\in\mathcal M_{T,\beta}.
\end{equation}

We next estimate the meta-regret
\(R_T^{\overline{\mathcal A}}(e_m)\).
The standard VAW bound gives
\begin{equation}
\label{eq:adaptive_m_final_meta_regret}
R_T^{\overline{\mathcal A}}(e_m)
\le
\frac12
+
\frac{K Y^2}{2}
\ln\!\left(1+\frac{W_T^2}{K}\right),
\qquad
W_T^2
:=
\sum_{m\in\mathcal M_{T,\beta}}
\sum_{t=1}^T
\bigl(\hat y_t^{(m)}\bigr)^2,
\qquad
K:=|\mathcal M_{T,\beta}|.
\end{equation}
For each \(m\in\mathcal M_{T,\beta}\),
\[
\sum_{t=1}^T
\bigl(\hat y_t^{(m)}\bigr)^2
\le
2\sum_{t=1}^T
\bigl(\hat y_t^{(m)}-y_t\bigr)^2
+
2\sum_{t=1}^T y_t^2
=
4R_T^{\mathcal A_m}(0)
+
4TY^2.
\]
For the zero comparator sequence \(f_t\equiv0\), we have
\(P=0\), and the approximation term vanishes. Hence
Corollary~\ref{cor:rough_bound_subspace} yields
\[
R_T^{\mathcal A_m}(0)
=
O\!\left(
(\ln T)^2
+
m\ln\!\left(1+\frac{T}{m}\right)
\right).
\]
It follows that
\[
W_T^2
=
O\!\left(
\sum_{m\in\mathcal M_{T,\beta}}
\left[
(\ln T)^2
+
T
+
m\ln\!\left(1+\frac{T}{m}\right)
\right]
\right).
\]
Since
$
K
=
|\mathcal M_{T,\beta}|
=
O(\ln T),
$
and
\[
\ln\!\left(1+\frac{T}{m}\right)
=
O(\ln T),
\qquad
\sum_{m\in\mathcal M_{T,\beta}}m
=
\sum_{j=0}^{\left\lfloor
    \frac{\log_2T}{\beta+1}
\right\rfloor}2^j
=
O\left(T^{\frac1{\beta+1}}\right),
\]
we obtain
\[
W_T^2
=
O\left(
(\ln T)^3
+
T\ln T
+
T^{\frac1{\beta+1}}\ln T
\right)
=
O(T\ln T).
\]
Substituting this into
\eqref{eq:adaptive_m_final_meta_regret}, and using
\(K=O(\ln T)\), we conclude that
\begin{equation}
\label{eq:meta_bound}
R_T^{\overline{\mathcal A}}(e_m)
=
O\!\left((\ln T)^2\right),
\qquad
m\in\mathcal M_{T,\beta}.
\end{equation}

Combining \eqref{eq:adaptive_m_general_fixed_m2} and
\eqref{eq:meta_bound}, we obtain
\begin{align}
R_T^{\overline{\mathcal A}}(f_{1:T})
&=
O\left(
(\ln T)^2
+
\min_{m\in\mathcal M_{T,\beta}}\Phi_P(m)
\right),
\label{eq:intermediate_bound}\\
\Phi_P(m)
&=
\sqrt{mTP}
+
m\ln\!\left(1+\frac{T}{m}\right)
+
Tm^{-\beta}.
\label{eq:Phi_P_definition}
\end{align}

To compare the dyadic grid \(\mathcal M_{T,\beta}\) with the
continuous range, put
$
M_{T,\beta}:=T^{\frac{1}{\beta+1}}.
$
For any \(m\in[1,M_{T,\beta}]\), let
$
\underline m:=2^{\lfloor\log_2m\rfloor}
\in\mathcal M_{T,\beta}.
$
Then
$
\frac m2<\underline m\le m.
$
The first two terms in \eqref{eq:Phi_P_definition} are
increasing in \(m\), while the last one is decreasing.
Therefore
\[
\sqrt{\underline m\,TP}
\le
\sqrt{mTP},
\qquad
\underline m\ln\!\left(1+\frac{T}{\underline m}\right)
\le
m\ln\!\left(1+\frac{T}{m}\right),
\qquad
T\underline m^{-\beta}
\le
2^\beta Tm^{-\beta}.
\]
Hence
$
\Phi_P(\underline m)
\le
C_\beta\Phi_P(m)$,
$C_\beta:=2^\beta.
$
Since this holds for every \(m\in[1,M_{T,\beta}]\), we obtain
\[
\min_{m\in\mathcal M_{T,\beta}}\Phi_P(m)
\le
C_\beta
\inf_{m\in[1,M_{T,\beta}]}\Phi_P(m).
\]

It remains to minimize \(\Phi_P(m)\) over
\([1,M_{T,\beta}]\).
Put
\[
D(T,P)
=
T^{\frac{\beta+1}{2\beta+1}}
P^{\frac{\beta}{2\beta+1}},
\qquad
S(T)
=
T^{\frac{1}{\beta+1}}
(\ln T)^{\frac{\beta}{\beta+1}}.
\]

Assume first that
\[
P\ge P_0:=T^{-\frac{\beta}{\beta+1}}
(\ln T)^{\frac{2\beta+1}{\beta+1}},
\]
which corresponds to the regime where the dynamic term dominates.
To balance the first and the third terms in \(\Phi_P(m)\), set
\[
m_d=
(2R_f)^{\frac{1}{2\beta+1}}
\left(\frac{T}{P}\right)^{\frac{1}{2\beta+1}}:
\qquad
\sqrt{m_dTP}
\propto
Tm_d^{-\beta}
\propto
D(T,P).
\]
Since \(P\le 2R_fT\), we have \(m_d\ge1\). On the other hand,
\[
m_d
\le
(2R_f)^{\frac{1}{2\beta+1}}
\left(\frac{T}{P_0}\right)^{\frac{1}{2\beta+1}}
=
(2R_f)^{\frac{1}{2\beta+1}}
\left(\frac{T}{\ln T}\right)^{\frac{1}{\beta+1}}
=
M_{T,\beta}
\frac{(2R_f)^{\frac{1}{2\beta+1}}}
     {(\ln T)^{\frac{1}{\beta+1}}}.
\]
Thus \(m_d\in[1,M_{T,\beta}]\) for all sufficiently large \(T\).
Moreover, the second term is estimated as
\[
m_d\ln\!\left(1+\frac{T}{m_d}\right)
=
O(m_d\ln T)
=
O\!\left(
\left(\frac{T}{P}\right)^{\frac{1}{2\beta+1}}
\ln T
\right)
=
O(D(T,P)),
\]
where the last estimate follows from the assumption \(P\ge P_0\).
Hence
\[
\inf_{1\le m\le M_{T,\beta}}\Phi_P(m)
\le
\Phi_P(m_d)
=
O(D(T,P)).
\]

Assume now that \(P\le P_0\),
which corresponds to the regime where the static term dominates.
Since
\[
m\ln\!\left(1+\frac{T}{m}\right)
=
O(m\ln T)
\qquad\text{uniformly in }m\in[1,M_{T,\beta}],
\]
we balance the last two terms in \eqref{eq:Phi_P_definition}
by setting
\[
m_s=
\left(\frac{T}{\ln T}\right)^{\frac{1}{\beta+1}}=
\frac{M_{T,\beta}}
     {(\ln T)^{\frac{1}{\beta+1}}}:
\qquad
m_s\ln T
=
Tm_s^{-\beta}
=
S(T).
\]
We have \(m_s\in[1,M_{T,\beta}]\) for all sufficiently large
\(T\), and
$
\sqrt{m_sTP}
\le
\sqrt{m_sTP_0}
=
S(T).
$
Thus
\[
\inf_{1\le m\le M_{T,\beta}}\Phi_P(m)
\le
\Phi_P(m_s)
=
O\left(S(T)\right).
\]

Combining the two regimes, we obtain
\[
\min_{m\in\mathcal M_{T,\beta}}\Phi_P(m)
=
O\left(
D(T,P)+S(T)
\right).
\]
Therefore, from \eqref{eq:intermediate_bound},
\[
R_T^{\overline{\mathcal A}}(f_{1:T})
=
O\left(
D(T,P)+S(T)+(\ln T)^2
\right).
\]
Since \((\ln T)^2=O(S(T))\), this yields the stated bound.
\end{proof}

\medskip
\noindent\emph{Algorithm summary.}
The forecasting procedure described above can be summarized as follows.

\begin{enumerate}

\item For a given finite-dimensional subspace \(V_m\subset\mathcal H\), construct
an orthonormal basis \(g_1,\dots,g_m\), and compute the feature vector
$
z_t=\Psi_{V_m}(x_t)
$
according to \eqref{eq:feature_map_subspace}.

\item For each nonzero discount factor \(\gamma\in S_\gamma\), where
\(S_\gamma\) is defined from the grid \eqref{eq:grid}, run the DVAW update
\eqref{eq:DVAW_update}. Denote the resulting predictions by
\(\hat y_{t,1},\dots,\hat y_{t,N-1}\), and set
\(\hat y_{t,0}=\tilde y_t\). Form the meta-feature vector
$
\mathbf z_t
=
(\hat y_{t,0},\hat y_{t,1},\dots,\hat y_{t,N-1}).
$
Apply the standard VAW forecaster \eqref{eq:standard_VAW} with
\(p=N\) and \(q_t=\mathbf z_t\). This produces the VE-DVAW prediction
\(\hat y_t^{(m)}\).

\item In the fast approximation regime, use the dimension \(m=m(T)\) specified
in Theorem~\ref{thm:fast_regime_general}. In the slow approximation regime,
run the
preceding procedure for every \(m\in\mathcal M_{T,\beta}\) and combine the
predictions \(\hat y_t^{(m)}\) by the standard VAW forecaster
\eqref{eq:standard_VAW}, with \(p=|\mathcal M_{T,\beta}|\),
\(q_t=(\hat y_t^{(m)})_{m\in\mathcal M_{T,\beta}}\), and \(\lambda=1\), as
specified in Theorem~\ref{thm:adaptive_over_m_slow_general}.

\item After observing \(y_t\), update all the corresponding DVAW and VAW
forecasters and proceed to the next round.

\end{enumerate}

\begin{remark}
The closest result to the present work is the bound in expectation
\begin{equation} \label{eq:universal_bound_in_expectation}
  \mathbb E R_T(f_{1:T})
        =
        O\left(
        T^{2/3}\bigl(P_T^\mathcal H(f_{1:T})\bigr)^{1/3}
        + \sqrt{T}\log T
        \right),
\end{equation}
obtained in \cite{Rokhlin2025hierarchical} using the random feature approach.
This estimate is universal with respect to the deterministic approximation
properties of the kernel: it does not involve a rate for \(\mathcal E(V_m)\).
In contrast, the bounds obtained in Theorems~4 and~5 are deterministic, but
kernel-dependent.

In the fast approximation regime, ignoring logarithmic factors, Theorem~4 gives
\[
        R_T(f_{1:T})
        =
        \widetilde O\left(\sqrt{T P_T^\mathcal H(f_{1:T})}\right),
\]
which is better than \eqref{eq:universal_bound_in_expectation} for $P_T^\mathcal H(f_{1:T})=o(T)$.
To make the comparison in the slow approximation regime clearer, assume that $P_T^\mathcal H(f_{1:T})\asymp T^\rho$, $0\le \rho<1 $. 
Then Theorem~5 gives
\begin{equation} \label{eq:simplified_deterministic_bound}
        R_T(f_{1:T})
        =
        \widetilde O\left(
        T^{\frac{\beta+1}{2\beta+1}}
        (P_T^\mathcal H(f_{1:T}))^{\frac{\beta}{2\beta+1}}
        +
        T^{\frac1{\beta+1}}
        \right)
        =
        \widetilde O\left(
        T^{\frac{\beta+1+\rho\beta}{2\beta+1}}
        +
        T^{\frac1{\beta+1}}
        \right).
\end{equation}
On the other hand, \eqref{eq:universal_bound_in_expectation} gives
\begin{equation} \label{eq:simplified_random_feature_bound}
        \mathbb E R_T(f_{1:T})
        =
        \widetilde O\left(
        T^{\frac{2+\rho}{3}}+\sqrt T
        \right).
\end{equation}
It is easy to see that both terms in the deterministic bound \eqref{eq:simplified_deterministic_bound} are better than those in the random-feature bound \eqref{eq:simplified_random_feature_bound} precisely when $\beta>1$. For $\beta<1$ both terms in the random-feature bound are better. 

Consequently, the two approaches are complementary. The random-feature method
provides a universal expected guarantee. The deterministic subspace method provides
kernel-dependent deterministic guarantees, and it is advantageous when the RKHS
admits sufficiently fast deterministic approximation.

The mentioned abstract approximation regimes above are realized by the concrete kernels considered below. The Gaussian kernel and the analytic dot-product kernels fall into the fast approximation regime: see Section~\ref{sec:explicit_feature_expansions}. The slow regime with \(\beta=\nu/d\) is realized for Mat\'ern kernels
when \(\nu\neq1\); for \(\nu=1\), the same polynomial approximation
order holds with an additional logarithmic factor: see
Section~\ref{sec:point_based_kernel_sections}.
\end{remark}

\begin{remark}
\label{rem:computational_complexity}

For a fixed \(m\)-dimensional subspace \(V_m\), one DVAW forecaster requires
\(O(m^2)\) arithmetic operations and \(O(m^2)\) memory per round. The VE-DVAW
forecaster runs $N=O(\ln T)$
such DVAW forecasters and combines their predictions by an additional
\(N\)-dimensional VAW forecaster, whose per-round time and memory requirements
are \(O(N^2)\).

Denote by \(c_i\) the cost of evaluating the feature function \(g_i\), and
assume that these costs are uniformly bounded, \(c_i=O(1)\). Then computing
the feature vector $\bigl(g_1(x_t),\ldots,g_m(x_t)\bigr)$
costs
\[
        C_m:=\sum_{i=1}^m c_i=O(m).
\]
This vector is computed once and shared by all DVAW forecasters corresponding
to the different discount factors. Possible
preprocessing costs required to construct the bases are not included.

Consequently, for a fixed \(V_m\), the per-round time requirement of VE-DVAW is
$
        O(C_m+Nm^2+N^2),
$
whereas its memory requirement is
$
        O(m+Nm^2+N^2).
$
Since \(C_m=O(m)\) and \(N=O(\ln T)\), both quantities are bounded by
\[
        O\!\left(m^2\ln T+(\ln T)^2\right).
\]
This gives the first line in Table~\ref{tab:complexity}, which summarizes the
resulting per-round complexity bounds.

\begin{table}[t]
\centering
\caption{Per-round time and memory complexity of the forecasters used in the
fixed-subspace, fast-approximation, and slow-approximation regimes.}
\label{tab:complexity}
\begin{tabular}{lc}
\hline
Algorithm
&
Time and memory per round
\\
\hline
VE-DVAW on a fixed \(V_m\)
&
\(O\!\left(m^2\ln T+(\ln T)^2\right)\)
\\[1mm]

Fast approximation regime, Theorem~4
&
\(O\!\left((\ln T)^{1+2/\alpha}+(\ln T)^2\right)\)
\\[1mm]

Slow approximation regime, Theorem~5
&
\(O\!\left(T^{\frac{2}{\beta+1}}\ln T\right)\)
\\
\hline
\end{tabular}
\end{table}

In the fast approximation regime of Theorem~4,
$m=O\!\left((\ln T)^{1/\alpha}\right)$, which gives the
second line of Table~\ref{tab:complexity}.

In the slow approximation regime of Theorem~5, one VE-DVAW forecaster is run
for every dimension in the dyadic grid
$
        \mathcal M_{T,\beta}
        =
        \{2^j:0\le j\le J_{T,\beta}\},
$
where
$
        J_{T,\beta}
        =
        \left\lfloor
            \frac{\log_2T}{\beta+1}
        \right\rfloor.
$
Since
\[
        \sum_{m\in\mathcal M_{T,\beta}}m^2
        =
        \sum_{j=0}^{J_{T,\beta}}4^j
        =
        \frac{4^{J_{T,\beta}+1}-1}{3}
        =
        O\left(T^{\frac{2}{\beta+1}}\right),
\]
the DVAW updates in all the VE-DVAW ensembles require
\[
        O\!\left(
        N\sum_{m\in\mathcal M_{T,\beta}}m^2
        \right)
        =
        O\left(
        T^{\frac{2}{\beta+1}}\ln T
        \right)
\]
time and memory per round.

Besides the DVAW updates, the algorithm of Theorem~5 includes one internal
\(N\)-dimensional VAW forecaster for each
\(m\in\mathcal M_{T,\beta}\), as well as one top-level VAW forecaster of
dimension \(|\mathcal M_{T,\beta}|\). Therefore its total per-round time
requirement is
\[
O\!\left(
N\sum_{m\in\mathcal M_{T,\beta}}m^2
+
|\mathcal M_{T,\beta}|N^2
+
|\mathcal M_{T,\beta}|^2
+
\sum_{m\in\mathcal M_{T,\beta}}C_m
\right).
\]
Here the four terms correspond, respectively, to all DVAW updates, the internal
VAW forecasters within the VE-DVAW ensembles, the top-level VAW forecaster, and
the computation of the feature vectors.
Since \(C_m=O(m)\) and the grid is dyadic,
\[
\sum_{m\in\mathcal M_{T,\beta}}C_m
=
O\left(T^{\frac{1}{\beta+1}}\right).
\]
Moreover,
$
|\mathcal M_{T,\beta}|N^2
=
O\left((\ln T)^3\right)$,
$|\mathcal M_{T,\beta}|^2
=
O\left((\ln T)^2\right).
$
Thus the total per-round time requirement is
\[
O\left(
T^{\frac{2}{\beta+1}}\ln T
+
(\ln T)^3
+
T^{\frac{1}{\beta+1}}
\right)
=
O\left(
T^{\frac{2}{\beta+1}}\ln T
\right).
\]

The corresponding memory requirement is bounded by the same expression with
the feature-evaluation term
\(\sum_{m\in\mathcal M_{T,\beta}}C_m\) replaced by
$
        O\!\left(
        \sum_{m\in\mathcal M_{T,\beta}}m
        \right)
        =
        O\left(T^{\frac{1}{\beta+1}}\right).  
$
Consequently, both the per-round time and memory requirements are bounded by
$
O\left(
T^{\frac{2}{\beta+1}}\ln T
\right).
$
This gives the last line of Table~\ref{tab:complexity}.
\end{remark}

\subsection{A general orthogonal truncation scheme}
\label{subsec:explicit_orthogonal_truncation}

We now describe a general mechanism for constructing approximation spaces \(V_m\)
from an orthonormal expansion of the kernel. This will be used repeatedly in the
subsequent kernel-specific examples.

\begin{lemma}
\label{lem:explicit_feature_construction}
Let $g_j\in C(X)$, $j\in\mathbb N$ be linearly independent functions.
Assume that there exists a continuous function
$
k:X\times X\to \mathbb R
$
such that
\begin{equation}
\label{eq:explicit_feature_expansion_general}
k(x,y)=\sum_{j=1}^\infty g_j(x)g_j(y),
\qquad x,y\in X,
\end{equation}
where the series converges pointwise. Define
\[
\mathcal H_{\mathrm{fin}}
=
\left\{
\sum_{j=1}^\infty c_j g_j:\ c_j=0
\text{ for all but finitely many }j
\right\},
\]
and for
$
f=\sum_{j=1}^\infty c_j g_j, h=\sum_{j=1}^\infty d_j g_j
\in \mathcal H_{\mathrm{fin}}$ put
\begin{equation}
\label{eq:explicit_feature_inner_product_general}
\langle f,h\rangle_{\mathcal H_{\mathrm{fin}}}
=
\sum_{j=1}^\infty c_j d_j.
\end{equation}
Then:

\begin{enumerate}
\item[(i)] the inner product \eqref{eq:explicit_feature_inner_product_general} is well defined;
\item[(ii)] its completion \(\mathcal H\) is an RKHS on \(X\) with reproducing kernel \(k\);
\item[(iii)] \((g_j)_{j\ge1}\) is an orthonormal basis of \(\mathcal H\);
\item[(iv)] for every \(x\in X\), the series
$
\sum_{j=1}^\infty g_j(x)\,g_j
$
converges in \(\mathcal H\) to \(k(\cdot,x)\);
\item[(v)] the series \eqref{eq:explicit_feature_expansion_general} converges uniformly on \(X\times X\).
\end{enumerate}
\end{lemma}

\begin{proof} (i) Take \(f=\sum_{j=1}^\infty c_j g_j\in\mathcal H_{\mathrm{fin}}\). From the linear independence condition it follows that the coefficients $c_i$ are uniquely defined. Therefore the quantity \eqref{eq:explicit_feature_inner_product_general} is well-defined.
It is immediate that this bilinear form is symmetric and positive semidefinite. Moreover,
if $\langle f,f\rangle_{\mathcal H_{\mathrm{fin}}}=\sum_{j=1}^\infty c_j^2=0$, then \(c_j=0\) for all \(j\), hence \(f=0\). Thus
\eqref{eq:explicit_feature_inner_product_general} is indeed an inner product on
\(\mathcal H_{\mathrm{fin}}\).

(ii) From the Cauchy--Schwarz inequality 
\begin{align} \label{eq:evaluation_functional_bound}
|f(x)|
=
\left|\sum_{j=1}^\infty c_j g_j(x)\right|
\le
\left(\sum_{j=1}^\infty g_j(x)^2\right)^{1/2}
\left(\sum_{j=1}^\infty c_j^2\right)^{1/2}
\le
\kappa \|f\|_{\mathcal H_{\mathrm{fin}}},\quad \kappa=\sup_{x\in X}\sqrt{k(x,x)}<\infty
\end{align}
it follows that each evaluation functional is bounded on \(\mathcal H_{\mathrm{fin}}\).
Let \(\mathcal H\) be the completion of \(\mathcal H_{\mathrm{fin}}\) with respect to the norm
induced by \eqref{eq:explicit_feature_inner_product_general}. Consider a Cauchy sequence
\((f_n)\subset\mathcal H_{\mathrm{fin}}\) in \(\mathcal H\).
The estimate \eqref{eq:evaluation_functional_bound} implies that $f_n$ is a Cauchy sequence in \(C(X)\):
\[
\|f_n-f_m\|_\infty
\le
\kappa \|f_n-f_m\|_{\mathcal H}
\]
Thus, \(f_n\) converges uniformly to some $f\in C(X)$, and each element of \(\mathcal H\) can be identified with a continuous function on \(X\).
If \((f_n)\) and \((h_n)\) are two sequences from \(\mathcal H_{\mathrm{fin}}\) converging to the
same element of \(\mathcal H\), then
\[
\|f_n-h_n\|_\infty\le \kappa \|f_n-h_n\|_{\mathcal H}\to 0,
\]
so they have the same uniform limit. Hence this identification is well defined.

Fix \(x\in X\). The series
$
\sum_{j=1}^\infty g_j(x)\,g_j
$
is Cauchy in \(\mathcal H\). Indeed, for \(p<q\),
\[
\left\|
\sum_{j=p+1}^q g_j(x)\,g_j
\right\|_{\mathcal H}^2
=
\sum_{j=p+1}^q g_j(x)^2
\to 0,
\qquad p,q\to\infty.
\]
Since \(\mathcal H\) is complete, there exists an element \(k_x\in\mathcal H\) such that
\[
k_x=\sum_{j=1}^\infty g_j(x)\,g_j
\qquad\text{in }\mathcal H.
\]
For $
f=\sum_{j=1}^\infty c_j g_j\in\mathcal H_{\mathrm{fin}}
$
we have the reproducing property:
\[
\langle f,k_x\rangle_{\mathcal H}
=
\left\langle \sum_{j=1}^\infty c_j g_j,\sum_{j=1}^\infty g_j(x)\,g_j\right\rangle_{\mathcal H}
=
\sum_{j=1}^\infty c_j g_j(x)
=
f(x).
\]

Since \(\mathcal H_{\mathrm{fin}}\) is dense in \(\mathcal H\), for any \(f\in\mathcal H\) there exists
a sequence \(f_n\in\mathcal H_{\mathrm{fin}}\) such that \(f_n\to f\) in \(\mathcal H\).
By continuity of the inner product,
\[
f_n(x)=\langle f_n,k_x\rangle_{\mathcal H}\to \langle f,k_x\rangle_{\mathcal H}.
\]
On the other hand, by \eqref{eq:evaluation_functional_bound},
$
|f_n(x)-f(x)|\le \kappa \|f_n-f\|_{\mathcal H}\to 0.
$
Hence,
\[
f(x)=\langle f,k_x\rangle_{\mathcal H},
\qquad f\in\mathcal H,\ x\in X.
\]
Thus \(\mathcal H\) is an RKHS on \(X\), and its reproducing kernel \(K\) is determined by
$K(\cdot,x):=k_x$, $x\in X$. 
Equivalently,
\[
K(x,y)=k_y(x)=\langle k_y,k_x\rangle_{\mathcal H}.
\]
It remains to show that \(K=k\). But
\[\langle k_y,k_x\rangle_{\mathcal H} 
=
\left\langle
\sum_{j=1}^\infty g_j(y)\,g_j,
\sum_{j=1}^\infty g_j(x)\,g_j
\right\rangle_{\mathcal H}
=
\sum_{j=1}^\infty g_j(x)g_j(y)
=
k(x,y).
\]

(iii) By definition of the inner product,
$
\langle g_i,g_j\rangle_{\mathcal H}=\delta_{ij}$,
$i,j\ge 1$.
So \((g_j)_{j\ge1}\) is an orthonormal system in \(\mathcal H\).
Moreover, \((g_j)_{j\ge1}\) is an orthonormal basis of \(\mathcal H\) since  the linear span of
\((g_j)_{j\ge1}\) is dense in \(\mathcal H\) by construction. 

(iv) This is already proved, since $k(\cdot,x)=k_x$.

(v) For every \(x,y\in X\),
\[
\left|k(x,y)-\sum_{j=1}^n g_j(x)g_j(y)\right|
=
\left|\sum_{j>n} g_j(x)g_j(y)\right|
\le
\left(\sum_{j>n} g_j(x)^2\right)^{1/2}
\left(\sum_{j>n} g_j(y)^2\right)^{1/2}.
\]
Therefore
\[
\sup_{x,y\in X}
\left|k(x,y)-\sum_{j=1}^n g_j(x)g_j(y)\right|
\le
\sup_{z\in X}\sum_{j>n} g_j(z)^2
\to 0,
\]
where the last limit is implied by the Dini theorem.
\end{proof}

For \(m\in\mathbb N\), put
\[
V_m=\operatorname{span}\{g_1,\dots,g_m\},
\qquad
\widetilde k_m(x,y)=\sum_{j=1}^m g_j(x)g_j(y).
\]
Define the uniform truncation error
\begin{equation} \label{eq:uniform_truncation_error}
\varepsilon_m=\sup_{x,y\in X}|k(x,y)-\widetilde k_m(x,y)|.
\end{equation}
Then
\[
\|(I-\Pi_{V_m})k(\cdot,x)\|_{\mathcal H}^2
=
\|k(\cdot,x)\|_{\mathcal H}^2-\|\Pi_{V_m}k(\cdot,x)\|_{\mathcal H}^2
=
k(x,x)-\widetilde k_m(x,x)
\le
\varepsilon_m.
\]
Hence $\mathcal E(V_m)\le \sqrt{\varepsilon_m}$. Therefore
any uniform estimate on the kernel truncation error \(\varepsilon_m\)
yields, via Corollary~\ref{cor:rough_bound_subspace}, a regret bound for the
corresponding VE-DVAW forecaster:
\begin{equation}
\label{eq:explicit_truncation_rough_bound}
R_T^{\mathcal A(\lambda)}(f_{1:T})
=
O\!\left(
(\ln T)^2
+
\sqrt{mT\,P_T^{\mathcal H}(f_{1:T})}
+
m\ln\!\left(1+\frac{T}{m}\right)
+
T\sqrt{\varepsilon_m}
\right).
\end{equation}

\section{Kernels with explicit feature expansions}
\label{sec:explicit_feature_expansions}

The results of Section~\ref{sec:abstract_truncation_scheme} reduce the problem to the
construction of finite-dimensional subspaces \(V_m\subset\mathcal H\) with small approximation
error \(\mathcal E(V_m)\). One natural mechanism for doing so is to start from an explicit
orthonormal expansion of the kernel and then truncate it.

Recall that $X$ is a compact set with nonempty interior. Under this assumption, the monomials $x^\alpha$ are linearly independent on \(X\). Assume also that 
$X\subset B_r=\{x\in\mathbb R^d:\|x\|_2\le r\}$.

\subsection{Polynomial kernel}
\label{subsec:polynomial_kernel_explicit}

Consider the polynomial kernel of degree $q$:
\[
k_q(x,y)
:=
\left(1+\frac{\langle x,y\rangle}{\sigma^2}\right)^q=\sum_{s=0}^q \binom{q}{s}\sigma^{-2s}\langle x,y\rangle^s,
\qquad q\in\mathbb N,
\qquad x,y\in X.
\]
For each \(s=0,\dots,q\) the multinomial theorem (see \cite[Theorem~1.7.2]{Merris2003combinatorics}) gives
\begin{equation}
\label{eq:multinomial_expansion}
\langle x,y\rangle^s
=
\left(\sum_{i=1}^d x^{(i)}y^{(i)}\right)^s
=
\sum_{|\alpha|=s}\frac{s!}{\alpha!}\,x^\alpha y^\alpha,
\end{equation}
where $\alpha=(\alpha_1,\dots,\alpha_d)\in\mathbb N_0^d$,
\[
|\alpha|:=\alpha_1+\cdots+\alpha_d,
\qquad
\alpha!:=\alpha_1!\cdots \alpha_d!,
\qquad
x^\alpha:=\prod_{i=1}^d (x^{(i)})^{\alpha_i}.
\]

Substituting \eqref{eq:multinomial_expansion} into the previous formula, we obtain
\begin{align*}
k_q(x,y)
&=
\sum_{s=0}^q \binom{q}{s}\sigma^{-2s}
\sum_{|\alpha|=s}\frac{s!}{\alpha!}\,x^\alpha y^\alpha=
\sum_{|\alpha|\le q}
\binom{q}{|\alpha|}\frac{|\alpha|!}{\alpha!}\sigma^{-2|\alpha|}\,x^\alpha y^\alpha\\
&=
\sum_{|\alpha|\le q}
\frac{q!}{(q-|\alpha|)!\,\alpha!}\,
\sigma^{-2|\alpha|}\,
x^\alpha y^\alpha.
\end{align*}
Hence
\[
k_q(x,y)=\sum_{|\alpha|\le q} g_\alpha(x)g_\alpha(y), \qquad g_\alpha(x)
:=
\sqrt{\frac{q!}{(q-|\alpha|)!\,\alpha!}}\,
\sigma^{-|\alpha|}x^\alpha,
\qquad |\alpha|\le q.
\]
After an arbitrary enumeration of the finite family \((g_\alpha)_{|\alpha|\le q}\),
Lemma~\ref{lem:explicit_feature_construction} shows that this family is an orthonormal
basis of the RKHS \(\mathcal H_q\) generated by \(k_q\).
Its dimension equals to the number of multi-indices \(\alpha\in\mathbb N_0^d\) such that
\(|\alpha|\le q\). Equivalently, this is the number $m_q$ of nonnegative integer solutions of
\[
\alpha_1+\cdots+\alpha_d+\alpha_{d+1}=q,
\]
and therefore
\[
m_q=\#\{\alpha\in\mathbb N_0^d:\ |\alpha|\le q\}=\binom{d+q}{q},
\]
see \cite[Corollary~1.6.9]{Merris2003combinatorics}.

If \(m\ge m_q\) and the space \(V_m\) contains all basis functions \(g_\alpha\) with
\(|\alpha|\le q\), then \(V_m=\mathcal H_q\) and \(\varepsilon_m=0\).
Therefore \eqref{eq:explicit_truncation_rough_bound} yields
\[
R_T^{\mathcal A(\lambda)}(f_{1:T})
=
O\!\left(
(\ln T)^2
+
\sqrt{T\,P_T^{\mathcal H_q}(f_{1:T})}
\right),\qquad m\ge m_q.
\]
Thus, for polynomial kernels, the approximation error disappears completely once the truncation
level reaches \(m_q=\binom{d+q}{q}\).

\subsection{Gaussian kernel}
\label{subsec:gaussian_fast_case}

Consider the Gaussian kernel:
\begin{align}
k_\sigma(x,y)
&=
\exp\!\left(-\frac{\|x-y\|_2^2}{2\sigma^2}\right)=
\exp\!\left(-\frac{\|x\|_2^2}{2\sigma^2}\right)
\exp\!\left(-\frac{\|y\|_2^2}{2\sigma^2}\right)
\exp\!\left(\frac{\langle x,y\rangle}{\sigma^2}\right)\nonumber\\
&=
\exp\!\left(-\frac{\|x\|_2^2}{2\sigma^2}\right)
\exp\!\left(-\frac{\|y\|_2^2}{2\sigma^2}\right)
\sum_{s=0}^\infty\frac{\langle x,y\rangle^s}{\sigma^{2s}s!}\nonumber\\
&=\sum_{\alpha\in\mathbb N_0^d}
g_\alpha(x)\,g_\alpha(y), \quad g_\alpha(x)
:=
\frac{x^\alpha}{\sigma^{|\alpha|}\sqrt{\alpha!}}
\exp\!\left(-\frac{\|x\|_2^2}{2\sigma^2}\right),
\qquad x,y\in X. \label{eq:Gaussian_kernel_expansion}
\end{align}
where we used the multinomial expansion \eqref{eq:multinomial_expansion}. Enumerate the countable family \((g_\alpha)_{\alpha\in\mathbb N_0^d}\) layer by layer, in nondecreasing order
of the total degree \(|\alpha|\), with an arbitrary order within each layer.
By Lemma~\ref{lem:explicit_feature_construction}, this family is an orthonormal basis of the RKHS \(\mathcal H_\sigma\) generated by \(k_\sigma\).
As in Section~\ref{subsec:explicit_orthogonal_truncation}, let \(V_m\) denote the span of the first \(m\) functions in this enumeration.

For \(M\in\mathbb N_0\), define
\begin{equation}
\label{eq:gaussian_truncated_features}
\Psi_M(x)
=
\bigl(g_\alpha(x)\bigr)_{|\alpha|\le M}
\in\mathbb R^{m_M},
\qquad
m_M=\#\{\alpha\in\mathbb N_0^d:\ |\alpha|\le M\}
=
\binom{d+M}{M},
\end{equation}
the corresponding subspace $
V_{m_M}=\operatorname{span}\{g_\alpha:\ |\alpha|\le M\},
$
and the truncated kernel
\begin{equation}
\label{eq:gaussian_truncated_kernel}
\widetilde k_M(x,y)
=
\langle \Psi_M(x),\Psi_M(y)\rangle
=
\sum_{|\alpha|\le M} g_\alpha(x)g_\alpha(y).
\end{equation}
Grouping the terms according to the total degree \(|\alpha|=s\) as in  \eqref{eq:Gaussian_kernel_expansion}, we obtain
\begin{align*}
k_\sigma(x,y)-\widetilde k_M(x,y)
&=
\exp\!\left(-\frac{\|x\|_2^2}{2\sigma^2}\right)
\exp\!\left(-\frac{\|y\|_2^2}{2\sigma^2}\right)
\left(
e^{\langle x,y\rangle/\sigma^2}
-
\sum_{s=0}^M \frac{\langle x,y\rangle^s}{\sigma^{2s}s!}
\right).
\end{align*}

By Taylor's theorem, for some $\xi$ with $|\xi|\le |z|$ we have
\[
\left|
e^z-\sum_{s=0}^M \frac{z^s}{s!}
\right|
=
\frac{|z|^{M+1}}{(M+1)!}e^{\xi}
\le
\frac{|z|^{M+1}}{(M+1)!}e^{|z|}.
\]
Applying this with \(z=\langle x,y\rangle/\sigma^2\), we get (see also \cite{Cotter2011explicit})
\begin{align*}
|k_\sigma(x,y)-\widetilde k_M(x,y)|
&\le
\exp\!\left(-\frac{\|x\|_2^2}{2\sigma^2}\right)
\exp\!\left(-\frac{\|y\|_2^2}{2\sigma^2}\right)
\frac{|\langle x,y\rangle|^{M+1}}{\sigma^{2M+2}(M+1)!}
e^{|\langle x,y\rangle|/\sigma^2}\\
&\le
\frac{1}{(M+1)!}
\left(\frac{\|x\|_2\|y\|_2}{\sigma^2}\right)^{M+1}.
\end{align*}
Thus,
\[\varepsilon_{m_M}=
\sup_{x,y\in X}|k_\sigma(x,y)-\widetilde k_M(x,y)|
\le\frac{b^{M+1}}{(M+1)!},\qquad b:=\frac{r^2}{\sigma^2}.\]
Using Stirling's lower bound \(n!\ge (n/e)^n\), we get
\[
\mathcal E(V_{m_M})^2
\le\varepsilon_{m_M}
\le
\left(\frac{be}{M+1}\right)^{M+1}
\le
2^{-(M+1)},
\qquad M+1\ge 2be.
\]

The sequence
\[
m_M=\binom{d+M}{M}=
\frac{(M+1)(M+2)\cdots(M+d)}{d!},
\qquad M=0,1,\dots,
\]
where \(m_0=1\), is strictly increasing, and tends to \(+\infty\).
Hence for every \(m\ge1\) there exists a unique integer \(M=M(m)\ge0\) such that
\[
m_M\le m<m_{M+1}.
\]
Since \(V_{m_M}\subset V_m\), we have
$
\mathcal E(V_m)\le \mathcal E(V_{m_M}).
$

If \(M\ge d\), then $m_M\le 2^d M^d/d!$. Hence,
\begin{equation}
  \label{eq:m_M_bound}
m<m_{M+1}\le \frac{2^d}{d!}(M+1)^d,\quad M+1\ge d.
\end{equation}

Therefore, for $M+1\ge \max\{d,2be\}$,
we obtain
\begin{align*}
\mathcal E(V_m)
&\le
2^{-(M+1)/2}=\exp\left(-(M+1)\frac{\ln 2}{2} \right)
\\
&\le
\exp\!\left(
-\frac{\ln 2}{2}\left(\frac{d!}{2^d}\right)^{1/d} m^{1/d}
\right)=e^{-C_2 m^{1/d}},\qquad C_2=\frac{\ln 2}{4}(d!)^{1/d}.
\end{align*}
The above bound holds for all $m\ge m_\ast:=m_{\,\lceil \max\{d,2be\}\rceil-1}$.
Since the set
\(\{1,\dots,m_\ast-1\}\) is finite, we may choose
\[
C_1=
\max\left\{
1,\;
\max_{1\le m<m_\ast}
\mathcal E(V_m) e^{C_2 m^{1/d}} \right\}.
\]
It follows that
\[
\mathcal E(V_m)\le C_1 e^{-C_2 m^{1/d}},
\qquad m\ge 1.
\]
Hence the Gaussian kernel satisfies the assumptions of
Theorem~\ref{thm:fast_regime_general} with \(\alpha=1/d\).
Consequently, with
\[
m(T)=\left\lceil \left(\frac{\ln T}{C_2}\right)^d\right\rceil,
\]
the VE-DVAW forecaster satisfies
\begin{equation}
  \label{eq:gaussian_kernel_bound}
R_T^{\mathcal A(\lambda)}(f_{1:T})
=
O\!\left(
\sqrt{T\,P_T^{\mathcal H_\sigma}(f_{1:T})}\,(\ln T)^{d/2}
+
(\ln T)^{d+1}
\right).
\end{equation}

\begin{remark}
A closely related bound 
\[
\sum_{t=1}^T (\hat y_t-y_t)^2-(f_t(x_t)-y_t)^2
=
O\!\left(
(\log T)^{\frac{d+1}{2}}
\sqrt{T\,P_T^{\mathcal H_\sigma}(f_{1:T})}
\;\vee\;
(\log T)^{\frac{d+1}{2}}
\right)
\]
was claimed in \cite[Theorem 13]{Baby2021dynamic}. Their approach is based on strongly adaptive aggregation of PKAWV base learners, introduced in \cite{Jezequel2019efficient}. However, it assumes a known constant $B$ such that
\[
\|f_t\|_{\mathcal H_\sigma}\le B,
\qquad
|y_t|\le B,
\qquad t=1,\dots,T.
\]
Thus their guarantee relies on an a priori uniform control of the whole comparator sequence. Moreover, we expect that in the static case
$P_T^{\mathcal H_\sigma}(f_{1:T})=0$
this bound should recover the order $O((\log T)^{d+1})$ from \cite{Jezequel2019efficient}, so the last term may need a slight correction.
\end{remark}

\subsection{Analytic dot-product kernels}
\label{subsec:dot_product_kernels_fast}

We next consider dot-product kernels of the form
$
k_f(x,y)=f(\langle x,y\rangle)$,
$x,y\in X,
$
where 
\begin{equation}
\label{eq:power_series_dot_product}
f(t)=\sum_{n=0}^\infty a_n t^n,
\qquad a_n\ge 0.
\end{equation}
For notational simplicity, assume in what follows that \(a_n>0\) for all \(n\). 
Zero coefficients can be handled by omitting the corresponding zero basis
functions.

Assume that the radius of convergence of \eqref{eq:power_series_dot_product} is strictly larger than
\(r^2\). Since
\[
|\langle x,y\rangle|
\le
\|x\|_2\|y\|_2
\le
r^2,
\qquad x,y\in X,
\]
the series 
\[
k_f(x,y)=\sum_{n=0}^\infty a_n\langle x,y\rangle^n,
\qquad x,y\in X
\]
converges absolutely and uniformly.
This class includes, for example, Vovk's infinite polynomial kernel, and the exponential dot-product kernel (see \cite[Section~3.1]{KarKarnick2012random}):
\[ k(x,y)=\frac{1}{1-\langle x,y\rangle} = \sum_{n=0}^\infty \langle x,y\rangle^n,\qquad k(x,y)=\exp\!\left(\frac{\langle x,y\rangle}{\sigma^2}\right) = \sum_{n=0}^\infty \frac{\langle x,y\rangle^n}{n!\sigma^{2n}}. \]
Note that if \(X\) is a ball, then by a unit-ball version of Schoenberg's theorem
(see \cite[Theorem~1]{KarKarnick2012random} and \cite[Theorem~2]{Schoenberg1942positive})
every positive definite dot-product kernel on \(X\) is generated by a function \(f\)
that is analytic and has nonnegative Maclaurin coefficients. 

Using the multinomial expansion \eqref{eq:multinomial_expansion}, we obtain
\begin{align*}
k_f(x,y)
&=
\sum_{n=0}^\infty a_n \langle x,y\rangle^n
=
\sum_{n=0}^\infty a_n
\sum_{|\alpha|=n}\frac{n!}{\alpha!}\,x^\alpha y^\alpha=
\sum_{\alpha\in\mathbb N_0^d}
a_{|\alpha|}\frac{|\alpha|!}{\alpha!}\,x^\alpha y^\alpha.
\end{align*}
Hence
\[ k_f(x,y)
=
\sum_{\alpha\in\mathbb N_0^d} g_\alpha(x)g_\alpha(y),
\qquad
g_\alpha(x)
:=
\sqrt{a_{|\alpha|}\frac{|\alpha|!}{\alpha!}}\;x^\alpha.\]
Enumerate the functions \(g_\alpha\) in nondecreasing order of \(|\alpha|\), with an arbitrary order within each layer.
Then Lemma~\ref{lem:explicit_feature_construction} shows that this family is an orthonormal basis
of the RKHS \(\mathcal H_f\) generated by \(k_f\). Similarly to the Gaussian case, for \(M\in\mathbb N_0\) we define $\Psi_M$ and the truncated kernel as in \eqref{eq:gaussian_truncated_features} and \eqref{eq:gaussian_truncated_kernel}. Let $V_{m_M}=\operatorname{span}\{g_\alpha:\ |\alpha|\le M\}$ be the corresponding subspace.

Grouping the terms according to the total degree \(|\alpha|=n\), we obtain
$
\widetilde k_M(x,y)
=
\sum_{n=0}^M a_n\langle x,y\rangle^n.
$
Therefore
\[
k_f(x,y)-\widetilde k_M(x,y)
=
\sum_{n=M+1}^\infty a_n\langle x,y\rangle^n,
\]
and hence
\begin{equation}
\label{eq:dot_product_tail_basic}
\varepsilon_{m_M}
=
\sup_{x,y\in X}|k_f(x,y)-\widetilde k_M(x,y)|
\le \sum_{n=M+1}^\infty a_n |\langle x,y\rangle|^n\le
\sum_{n=M+1}^\infty a_n r^{2n}.
\end{equation}

We now use the assumption that the radius of convergence of \eqref{eq:power_series_dot_product}
is strictly larger than \(r^2\). Fix any \(\rho>r^2\) smaller than the radius of convergence of
\eqref{eq:power_series_dot_product}. By the Cauchy--Hadamard formula,
$
\limsup_{n\to\infty} a_n^{1/n}\le \rho^{-1}.$
Hence
\[
\limsup_{n\to\infty} (a_n r^{2n})^{1/n}
\le
\frac{r^2}{\rho}
<1.
\]
Choose $q\in (r^2/\rho,1)$.
Then there exists \(N_0\in\mathbb N\) such that $a_n r^{2n}\le q^n$, $n\ge N_0$.
Substituting this into \eqref{eq:dot_product_tail_basic}, we obtain 
\begin{equation} \label{eq:dot_product_tail}
\mathcal E(V_{m_M})^2
\le\varepsilon_{m_M}
\le
\sum_{n=M+1}^\infty q^n
=
\frac{q^{M+1}}{1-q}=\frac{1}{1-q}
\exp\!\left(-(M+1)\ln\frac{1}{q}\right),\qquad M\ge N_0.
\end{equation}

Now fix the same enumeration of the basis functions as above, and for each \(m\ge1\)
let \(V_m\) be the span of the first \(m\) basis functions in this enumeration.
For every \(m\ge1\), let \(M=M(m)\) be the unique integer such that
\[
m_M\le m<m_{M+1}.
\]
Combining this with \eqref{eq:dot_product_tail} and repeating the same
argument as in the Gaussian case based on the estimate \eqref{eq:m_M_bound},
we obtain 
\begin{align*}
\mathcal E(V_m)
&\le
\mathcal E(V_{m_M})
\le
\frac{1}{\sqrt{1-q}}
\exp\!\left(-\frac{M+1}{2}\ln\frac{1}{q}\right)
\\
&\le
\frac{1}{\sqrt{1-q}}
\exp\!\left(
-\frac{\ln(1/q)}{2}\left(\frac{d!}{2^d}\right)^{1/d}m^{1/d}
\right),\qquad \text{if } M+1\ge \max\{d, N_0+1\}.
\end{align*}
It follows that there exist constants \(C_1,C_2>0\) such that
\[
\mathcal E(V_m)\le C_1 e^{-C_2 m^{1/d}},
\qquad m\ge1,
\]
where one may take
\[
C_2=\frac{\ln(1/q)}{4}(d!)^{1/d}.
\]
Hence analytic dot-product kernels satisfy the assumptions of
Theorem~\ref{thm:fast_regime_general} with \(\alpha=1/d\), which implies a bound of the same form as for Gaussian kernels with $\mathcal H_f$ instead of $\mathcal H_\sigma$: see \eqref{eq:gaussian_kernel_bound}. 

\section{Kernel expansions via the Mercer theorem}
\label{sec:mercer_expansions}
Section~\ref{sec:explicit_feature_expansions} relied on explicit orthonormal
feature expansions of the kernel. When such expansions are not available in a
tractable form, one can instead use the Mercer expansion associated with the
kernel and a reference measure on the domain. This is the approach we follow in
the present section.

\subsection{A general Mercer truncation scheme}
\label{subsec:general_mercer_truncation}
Let \(\mu\) be a Borel probability
measure on \(X\) with \(\operatorname{supp}\mu=X\), and let
$k:X\times X\to\mathbb R$
be a continuous symmetric positive definite kernel.
Consider the associated integral operator
\[
(T_\mu f)(x)
=
\int_X k(x,y)f(y)\,d\mu(y),
\qquad f\in L_2(X,\mu).
\]

Since \(k\) is continuous on \(X\times X\), the operator \(T_\mu\) is compact
on \(L_2(X,\mu)\); moreover, since \(k\) is symmetric and positive definite,
\(T_\mu\) is self-adjoint and positive
(see, e.g., \cite[Propositions~4.5 and~4.6]{Cucker2007learning}).
Hence, by the spectral theorem for compact self-adjoint operators
(cf. \cite[Theorem~4.4]{Cucker2007learning}), \(L_2(X,\mu)\) has an
orthonormal basis consisting of eigenfunctions of \(T_\mu\). We enumerate the
positive eigenvalues in nonincreasing order:
$\lambda_1\ge \lambda_2\ge\cdots>0$,
repeating them according to multiplicity, and denote the corresponding
orthonormal eigenfunctions by \((\phi_j)_{j\ge1}\). For notational simplicity we assume that there are infinitely many positive eigenvalues.

For each \(j\), since $\phi_j=\lambda_j^{-1}T_\mu\phi_j$,
the eigenfunction \(\phi_j\) can be chosen continuous on \(X\).
By Mercer's theorem \cite[Theorem 4.10]{Cucker2007learning},
\begin{equation}
\label{eq:mercer-expansion}
k(x,y)
=
\sum_{j=1}^\infty\lambda_j\phi_j(x)\phi_j(y),
\qquad x,y\in X,
\end{equation}
with uniform convergence on \(X\times X\). Put $g_j:=\sqrt{\lambda_j}\,\phi_j$,
$\lambda_j>0$. 
Then \eqref{eq:mercer-expansion} can be rewritten as
\[
k(x,y)
=
\sum_{j=1}^\infty g_j(x)g_j(y),
\qquad x,y\in X.
\]

Thus, once the Mercer expansion is already known, 
Lemma~\ref{lem:explicit_feature_construction} becomes applicable to the family
\((g_j)_{j\ge1}\). In particular, it identifies the RKHS \(\mathcal H_k\)
generated by \(k\) with the Hilbert space generated by \((g_j)_{j\ge1}\), and
shows that \((g_j)_{j\ge1}\) is an orthonormal basis of \(\mathcal H_k\). We also note that the fact that \((g_j)_{j\ge 1}\) is an orthonormal basis
of \(\mathcal H_k\) is classical: \cite[Theorem~4.12]{Cucker2007learning}.

For \(m\in\mathbb N\), define the truncated Mercer subspace
\[
V_m
=
\operatorname{span}\{g_j:\ 1\le j\le m\}
\subset \mathcal H_k,
\]
and the truncated Mercer kernel
\[
k_m(x,y)
=
\sum_{1\le j\le m}
\lambda_j\phi_j(x)\phi_j(y).
\]
It remains to bound the uniform truncation error $\varepsilon_m$, defined in \eqref{eq:uniform_truncation_error}. 

For this purpose we invoke the results of \cite{Takhanov2023speed}. Namely, if \(X\subset\mathbb R^d\) has Lipschitz boundary and
\(k\in C^{2p}(X\times X)\) for some integer \(p\ge 1\), then
by the comments after Theorems~1.1 and~1.2 in \cite{Takhanov2023speed}, one obtains
the simplified bounds
\begin{equation}
\label{eq:Takhanov_simplified}
\varepsilon_m
\le
C\,
\min\left\{
\left(\sum_{j>m}\lambda_j\right)^{\frac{p}{d+p}},
\left(\sum_{j>m}\lambda_j^2\right)^{\frac{p}{d+2p}}
\right\},
\qquad m\ge 1.
\end{equation}

Assume that the eigenvalues decay stretched-exponentially:
\begin{equation}
  \label{eq:stretched_exp_decay}
 \lambda_j\le C_\lambda e^{-c_\lambda j^\alpha},
\qquad j\ge 1, 
\end{equation}

with some constants \(C_\lambda,c_\lambda,\alpha>0\). We have
\[ \sum_{j>m} e^{-c_\lambda j^\alpha}\le I(m):=
\int_m^\infty e^{-c_\lambda x^\alpha}\,dx.\]
If \(0<\alpha<1\), then $I(m)= O(e^{-c m^\alpha})$ for every \(c\in(0,c_\lambda)\).
Indeed, L'Hospital's rule gives $I(m)=o(e^{-c m^\alpha})$, $m\to\infty$.
If \(\alpha\ge 1\), then  $I(m)=O(e^{-c_\lambda m^\alpha})$, since
\[
\lim_{m\to\infty}
\frac{\int_m^\infty e^{-c_\lambda x^\alpha}\,dx}{e^{-c_\lambda m^\alpha}}
=
\lim_{m\to\infty}
\frac{1}{c_\lambda \alpha m^{\alpha-1}}.
\]
Therefore, in all cases, there exists \(c>0\) such that
\[
\sum_{j>m}\lambda_j = O(e^{-c m^\alpha}),
\qquad
\sum_{j>m}\lambda_j^2 = O(e^{-2c m^\alpha}).
\]
Applying \eqref{eq:Takhanov_simplified}, we get
\[
\varepsilon_m
=
O\!\left(
\min\left\{
e^{-\frac{cp}{d+p}m^\alpha},
e^{-\frac{2cp}{d+2p}m^\alpha}
\right\}
\right).
\]
Since
\[
\frac{2p}{d+2p}-\frac{p}{d+p}
=
\frac{pd}{(d+2p)(d+p)}
>0,
\]
the second bound is stronger. Hence
\begin{equation}
\label{eq:epsilon_m_stretched_exp_general}
\mathcal E(V_m)\le\sqrt{\varepsilon_m}
=
O\!\left(
\exp\left(-\frac{cp}{d+2p} m^\alpha\right)
\right).
\end{equation}
Thus the Mercer truncation scheme falls into the fast approximation regime, and
Theorem~\ref{thm:fast_regime_general} yields
\begin{equation}
\label{eq:mercer_stretched_exp_final}
R_T^{\mathcal A(\lambda)}(f_{1:T})
=
O\!\left(
\sqrt{T\,P_T^{\mathcal H}(f_{1:T})}\,
(\ln T)^{\frac{1}{2\alpha}}
+
(\ln T)^{1+\frac1\alpha}
+
(\ln T)^2
\right).
\end{equation}

Assume now that the eigenvalues decay polynomially:
\(\lambda_j\le C_\lambda j^{-r}\) with \(r>1\). Then
\[
\sum_{j>m}\lambda_j=O(m^{1-r}),
\qquad
\sum_{j>m}\lambda_j^2=O(m^{1-2r}),
\]
and therefore \eqref{eq:Takhanov_simplified} gives
\[
\varepsilon_m
=
O\!\left(
\min\left\{
m^{-\frac{(r-1)p}{d+p}},
m^{-\frac{(2r-1)p}{d+2p}}
\right\}
\right).
\]
Since
\[
\frac{(2r-1)p}{d+2p}
-
\frac{(r-1)p}{d+p}
=
\frac{p(rd+p)}{(d+2p)(d+p)}
>0,
\]
the second bound is stronger, and hence
\begin{equation}
\label{eq:epsilon_m_poly_decay}
\mathcal E(V_m)\le \sqrt{\varepsilon_m}
=
O\!\left(m^{-\beta}\right),\qquad \beta=\frac{(2r-1)p}{2(d+2p)}.
\end{equation}
Thus under the stated assumptions the Mercer truncation scheme falls into the slow approximation regime. Applying Theorem~\ref{thm:adaptive_over_m_slow_general}, we get
\begin{equation}
\label{eq:mercer_poly_decay_final}
R_T^{\overline{\mathcal A}}(f_{1:T})
=O\left(
T^{\frac{\beta+1}{2\beta+1}}
\Bigl(P_T^{\mathcal H}(f_{1:T})\Bigr)^{\frac{\beta}{2\beta+1}}
+
T^{\frac{1}{\beta+1}}(\ln T)^{\frac{\beta}{\beta+1}}
\right),\qquad \beta=\frac{(2r-1)p}{2(d+2p)}.
\end{equation}

Assume that the boundary of \(X\) is sufficiently smooth.
For translation-invariant kernels of the form
\[
k(x,y)=\kappa(x-y),
\qquad x,y\in X,
\]
let \(\widehat\kappa\) denote the Fourier transform of \(\kappa\). By
\cite[Theorems~6.5 and~6.8]{Schaback2002approximation}, a polynomial decay of the Fourier transform,
\[
\widehat\kappa(\omega)\asymp (1+\|\omega\|_2)^{-(\tau+d)},
\qquad \|\omega\|_2\to\infty,
\]
implies the polynomial eigenvalue decay
$
\lambda_j=O\!\left(j^{-(\tau+d)/d}\right),
$
while an exponential decay of the Fourier transform,
\[
\widehat\kappa(\omega)\le C e^{-c\|\omega\|_2},
\qquad \|\omega\|_2\to\infty,
\]
implies the stretched-exponential eigenvalue decay
$
\lambda_j=O\!\left(e^{-c'j^{1/d}}\right).
$
In particular, for Gaussian kernels the latter implies
\eqref{eq:epsilon_m_stretched_exp_general}, and therefore leads to the same fast
approximation regime as in Section~\ref{sec:explicit_feature_expansions}. 

\subsection{Tensor-product analytic kernels on the hypercube}
Let
$
X=[-1,1]^d$,
$
\mu=\bigotimes_{i=1}^d \mu_i$,
$
\mu_i:=\frac12\,dx$,
and consider a tensor-product kernel
\[
k(x,y)=\prod_{i=1}^d k_i(x_i,y_i),
\qquad x,y\in X,
\]
where, for each \(i=1,\dots,d\), \(k_i\) is a continuous
kernel on \([-1,1]^2\), analytic in both variables.
Let \(T_i\) be the integral operator on \(L_2([-1,1],\mu_i)\) generated by
\(k_i\), and let
\[
T_i\phi_n^{(i)}=\lambda_n^{(i)}\phi_n^{(i)},
\qquad n\ge0,
\]
be its eigenpairs, with the eigenvalues ordered in nonincreasing order.
By the result of \cite{Little1984eigenvalues}, for each \(i\) there exist
constants \(C_i,c_i>0\) such that
\[
\lambda_n^{(i)}\le C_i e^{-c_i n},
\qquad n\ge0.
\]

The integral operator \(T\) corresponding to the product kernel \(k\) is the
tensor product \(T_1\otimes\cdots\otimes T_d\). Hence one may choose its
eigenfunctions in the product form
\[
\Phi_\alpha(x)=\prod_{i=1}^d \phi_{\alpha_i}^{(i)}(x_i),
\qquad
\alpha=(\alpha_1,\dots,\alpha_d)\in\mathbb N_0^d,
\]
with corresponding eigenvalues
$
\lambda_\alpha=\prod_{i=1}^d \lambda_{\alpha_i}^{(i)}.
$
Therefore
\[
\lambda_\alpha
\le
\prod_{i=1}^d C_i e^{-c_i\alpha_i}
=
C_\ast \exp\!\left(-\sum_{i=1}^d c_i\alpha_i\right)
\le
C_\ast e^{-c_\ast|\alpha|_1},
\]
where
$
C_\ast=\prod_{i=1}^d C_i$,
$
c_\ast=\min_{1\le i\le d} c_i$,
$
|\alpha|_1=\alpha_1+\cdots+\alpha_d$.

Let \((\lambda_j)_{j\ge1}\) be a nonincreasing rearrangement of the
eigenvalues \((\lambda_\alpha)_{\alpha\in\mathbb N_0^d}\), counting
multiplicities. Put
$
m_M=\#\{\alpha\in\mathbb N_0^d:\ |\alpha|_1\le M\}.
$
If \(|\alpha|_1\ge M+1\), then
\[
\lambda_\alpha
\le
C_\ast e^{-c_\ast(M+1)}.
\]
Hence all eigenvalues larger than \(C_\ast e^{-c_\ast(M+1)}\) must correspond
to multi-indices satisfying \(|\alpha|_1\le M\). There are only \(m_M\) such
multi-indices. Therefore, for every \(j>m_M\),
\[
\lambda_j\le C_\ast e^{-c_\ast(M+1)}.
\]
In particular, this is true for \(m_M<j\le m_{M+1}\).

On the other hand, using the inequality \eqref{eq:m_M_bound}, we get
\[
M+1
\ge
\left(\frac{d!}{2^d}\right)^{1/d}j^{1/d},\quad j\le m_{M+1},\quad M+1\ge d.
\]
Substituting this into the previous estimate, we obtain
\[
\lambda_j
\le
C_\ast
\exp\!\left(
-c_\ast\left(\frac{d!}{2^d}\right)^{1/d}j^{1/d}
\right)
\]
for all sufficiently large \(j\). Changing the constant \(C_\ast\), if necessary,
we may assume that this bound holds for all \(j\ge1\). Thus the
stretched-exponential eigenvalue decay condition \eqref{eq:stretched_exp_decay}
holds with $\alpha=1/d$.

Moreover, since the product kernel
\(k\) is \(C^\infty\) on \(X\times X\), the inequality \eqref{eq:epsilon_m_stretched_exp_general}
applies for every \(p\ge1\). 
Therefore the fast-regime result \eqref{eq:mercer_stretched_exp_final} applies
with \(\alpha=1/d\), and yields the same result as for Gaussian and dot-product kernels: 
\[
R_T^{\mathcal A(\lambda)}(f_{1:T})
=
O\!\left(
\sqrt{T\,P_T^{\mathcal H_k}(f_{1:T})}\,(\ln T)^{d/2}
+
(\ln T)^{d+1}
\right).
\]

\subsection{Mat\'ern kernels via Mercer truncation}
We now specialize the previous Mercer--Takhanov reduction to Mat\'ern kernels on
a bounded domain \(X\subset\mathbb R^d\) with sufficiently smooth boundary.
Let
\begin{equation}
  \label{eq:Mathern_kernel}
k_{\nu,\ell}(x,y)
=\kappa_{\nu,\ell}(x-y)
=
\frac{2^{1-\nu}}{\Gamma(\nu)}
\left(\frac{\sqrt{2\nu}}{\ell}\|x-y\|_2\right)^\nu
K_\nu\!\left(\frac{\sqrt{2\nu}}{\ell}\|x-y\|_2\right),
\qquad x,y\in X,
\end{equation}
be the restriction to \(X\times X\) of the Mat\'ern kernel on
\(\mathbb R^d\); see, e.g., \cite[Section~2.10]{Stein1999interpolation}. 
Here \(\nu>0\) is the smoothness parameter, \(\ell>0\) is the length-scale
parameter, and \(K_\nu\) is the modified Bessel function of the second kind.
Its Fourier transform satisfies
\begin{equation} 
  \label{eq:Matern_Fourier_asymptotics}
\widehat{\kappa}_{\nu,\ell}(\omega)
=
c_{\nu,\ell,d}
\left(\frac{2\nu}{\ell^2}+\|\omega\|_2^2\right)^{-(\nu+d/2)}
\asymp
\|\omega\|_2^{-(2\nu+d)},
\qquad \|\omega\|_2\to\infty.
\end{equation}

Assume that $\nu>1$, and fix an integer \(p\ge 1\) such that \(p<\nu\). Then
\[
|\omega^\alpha|\,\widehat{\kappa}_{\nu,\ell}(\omega)\in L_1(\mathbb R^d),
\qquad |\alpha|\le 2p,
\]
and hence \(k_{\nu,\ell}\in C^{2p}(X\times X)\). Furthermore,
applying the already mentioned \cite[Theorem~6.5]{Schaback2002approximation} yields $
\lambda_j =O( j^{-1-2\nu/d}).
$
Hence the polynomial eigenvalue-decay condition holds with
$
r=1+\frac{2\nu}{d},
$
and \eqref{eq:mercer_poly_decay_final} applies with
\begin{equation}
  \label{eq:beta_Matern_Mercer}
\beta
=
\frac{(2r-1)p}{2(d+2p)}
=
\frac{p(d+4\nu)}{2d(d+2p)},\qquad \textrm{if } 1\le p<\nu.
\end{equation}

Note that this argument applies to sufficiently smooth Mat\'ern kernels, and does not cover the case \(0<\nu\le1\). The rough regime \(0<\nu<1\) will be treated in the next section by the approximation scheme based on kernel sections.

\section{Approximation by subspaces spanned by kernel sections}
\label{sec:point_based_kernel_sections}

In the previous section, the approximation subspaces \(V_m\subset\mathcal H\)
were constructed via truncations of the Mercer expansion. We now discuss a
different route, based directly on kernel sections \(k(\cdot,z)\), \(z\in X\).

Throughout this section, let \(X\subset\mathbb R^d\) be a nonempty compact set,
and let \(k:X\times X\to\mathbb R\) be a continuous positive definite kernel
with RKHS \(\mathcal H\). The resulting construction is purely geometric: it is
governed by the kernel pseudometric induced by the RKHS norm, and does not
require an explicit description of the Mercer eigenfunctions.

\subsection{Kernel pseudometric and subspaces spanned by kernel sections}

For \(x,x'\in X\), define the kernel pseudometric
\[
\varrho(x,x')
=
\|k(\cdot,x)-k(\cdot,x')\|_{\mathcal H}
=
\bigl(k(x,x)-2k(x,x')+k(x',x')\bigr)^{1/2},
\]
induced by the canonical feature map \(x\mapsto k(\cdot,x)\) \cite{Steinwart2008support}. Given a finite set \(Z_m=\{z_1,\dots,z_m\}\subset X\), define the associated
subspace
\[
V_{Z_m}=\operatorname{span}\{k(\cdot,z_1),\dots,k(\cdot,z_m)\}\subset \mathcal H.
\]
Our goal is to estimate the approximation quantity
\[
\mathcal E(V_{Z_m})=
\sup_{x\in X}\|(I-\Pi_{V_{Z_m}})k(\cdot,x)\|_{\mathcal H}
\]
in terms of geometric properties of the set \(Z_m\) in the metric \(\varrho\).

\begin{lemma}
\label{lem:kernel-section_based_cover}
Assume that \(Z_m\subset X\) is an \(\varepsilon\)-net in the metric \(\varrho\), that is,
for every \(x\in X\) there exists \(z\in Z_m\) such that
$
\varrho(x,z)\le \varepsilon.
$
Then
$
\mathcal E(V_{Z_m})\le \varepsilon.
$
\end{lemma}

\begin{proof}
Fix \(x\in X\), and choose \(z\in Z_m\) such that \(\varrho(x,z)\le \varepsilon\).
Since \(k(\cdot,z)\in V_{Z_m}\)  and \(\Pi_{V_{Z_m}}k(\cdot,x)\) is
the best approximation to \(k(\cdot,x)\) in \(V_{Z_m}\), we have
\[
\|(I-\Pi_{V_{Z_m}})k(\cdot,x)\|_{\mathcal H}
\le
\|k(\cdot,x)-k(\cdot,z)\|_{\mathcal H}
=
\varrho(x,z)
\le
\varepsilon.
\]
Taking the supremum over \(x\in X\) yields the claim.
\end{proof}

Thus, approximation by subspaces spanned by kernel sections 
reduces to covering the set \(X\) in the kernel
pseudometric \(\varrho\). For $\varepsilon>0$, denote by
\[ N(\varepsilon,X,\varrho)
=
\min\Bigl\{
M\in\mathbb N:\ \exists z_1,\dots,z_M\in X
\text{ such that }
X\subset \bigcup_{i=1}^M \overline B_\varrho(z_i,\varepsilon)
\Bigr\}
\]
the $\varepsilon$-covering number of $X$.

\begin{corollary}
\label{cor:kernel-section_based_covering}
Assume that
\[
N(\varepsilon,X,\varrho)\le C_{\varrho}\,\varepsilon^{-q},
\qquad 0<\varepsilon\le \varepsilon_0,
\]
for some constants \(C_{\varrho},\varepsilon_0,q>0\). Then for every \(m\ge1\) there
exists a set \(Z_m\subset X\) with \(|Z_m|=m\) such that
\[
\mathcal E(V_{Z_m})=O(m^{-1/q}).
\]
\end{corollary}

\begin{proof}
Put $\varepsilon_m=(C_{\varrho}/m)^{1/q}$
for sufficiently large $m$ ensuring that $\varepsilon_m\le \varepsilon_0$. Then
\[
N(\varepsilon_m,X,\varrho)\le C_\varrho \varepsilon_m^{-q}=m.
\]
By definition of the covering number, there exists an \(\varepsilon_m\)-net
\(Z_m'\subset X\) such that
$
|Z_m'|\le m.
$
Enlarge $Z_m'$ if necessary so that $|Z_m|=m$, $Z_m'\subset Z_m$.
By Lemma~\ref{lem:kernel-section_based_cover},
\[
\mathcal E(V_{Z_m})\le \varepsilon_m
=
\left(\frac{C_\varrho}{m}\right)^{1/q}
=
O(m^{-1/q}).
\]
This proves the claim for all sufficiently large \(m\). Since changing finitely
many initial values of \(m\) does not affect the \(O(m^{-1/q})\) bound, the result
holds for all \(m\ge1\).
\end{proof}

The previous corollary reduces the approximation problem to estimating covering
numbers in the kernel pseudometric \(\varrho\). A convenient way to do this is to
compare \(\varrho\) with the Euclidean metric.

\begin{lemma}
\label{lem:rho_holder_covering}
Assume that \(X\subset\mathbb R^d\) is compact and that for some
\(\gamma\in(0,1]\) and \(C_\varrho>0\),
\[
\varrho(x,y)\le C_\varrho \|x-y\|_2^\gamma,
\qquad x,y\in X.
\]
Then for every \(m\ge1\)
there exists a subspace \(V_{Z_m}\) with \(|Z_m|=m\) such that
\[
\mathcal E(V_{Z_m})=O(m^{-\gamma/d}).
\]
\end{lemma}

\begin{proof}
Let \(\delta=(\varepsilon/C_\varrho)^{1/\gamma}\). If \(Z\subset X\) is a
\(\delta\)-net in the Euclidean metric, then for every \(x\in X\) there exists
\(z\in Z\) such that \(\|x-z\|_2\le \delta\), and hence
\[
\varrho(x,z)\le C_\varrho \|x-z\|_2^\gamma \le \varepsilon.
\]
Thus every Euclidean \(\delta\)-net is automatically an \(\varepsilon\)-net in
\((X,\varrho)\).

Since \(X\subset B_r\) for some \(r>0\), by the standard estimate
for the Euclidean ball (see, e.g., \cite[Corollary~4.2.13]{Vershynin2018high}),
\[
N(\delta,B_r,\|\cdot\|_2)\le \left(1+\frac{2r}{\delta}\right)^d.
\]
Therefore
\begin{align}
  \label{eq:covering_number_bound_Euclidean}
N(\delta,X,\|\cdot\|_2)
\le
N(\delta,B_r,\|\cdot\|_2)
\le
\left(1+\frac{2r}{\delta}\right)^d
\le
C_X\,\delta^{-d}
\end{align}
for all sufficiently small \(\delta>0\). Hence
\[
N(\varepsilon,X,\varrho)
\le
N\!\left((\varepsilon/C_\varrho)^{1/\gamma},X,\|\cdot\|_2\right)
\le
C_X C_\varrho^{d/\gamma}\,\varepsilon^{-d/\gamma}.
\]
The last claim now follows from Corollary~\ref{cor:kernel-section_based_covering}
with \(q=d/\gamma\).
\end{proof}

The general subspace reduction requires an orthonormal basis of \(V_{Z_m}\) in
\(\mathcal H\), which gives rise to the corresponding feature map.
The next lemma provides a standard construction based on the Gram matrix.

\begin{lemma}
\label{lem:gram_orthonormal_basis}
Let \(Z_m=\{z_1,\dots,z_m\}\subset X\), and put
\[
\phi_i:=k(\cdot,z_i)\in\mathcal H,
\qquad i=1,\dots,m.
\]
Consider the Gram matrix of the system \((\phi_i)_{i=1}^m\):
\[ G=\bigl(\langle \phi_i,\phi_j\rangle_{\mathcal H}\bigr)_{i,j=1}^m=\bigl(k(z_i,z_j)\bigr)_{i,j=1}^m,\]
and its spectral decomposition
\[
G=U\Lambda U^\top,\qquad \Lambda=\operatorname{diag}(\lambda_1,\dots,\lambda_{n_m},0,\dots,0).\]
Here \(\lambda_1,\dots,\lambda_{n_m}>0\) are the positive eigenvalues of \(G\), and
\(U=(U_{ij})_{i,j=1}^m\) is an orthogonal matrix. Define
\begin{equation}
  \label{eq:orthonormal_basis_V_m}
g_j=\frac{1}{\sqrt{\lambda_j}}
\sum_{i=1}^m U_{ij}\phi_i,
\qquad j=1,\dots,n_m.
\end{equation}
Then \(g_1,\dots,g_{n_m}\) form an orthonormal basis of
\[
V_{Z_m}=\operatorname{span}\{\phi_1,\dots,\phi_m\}.
\]
In particular,
$
\dim V_{Z_m}=n_m=\operatorname{rank}(G).
$
\end{lemma}

\begin{proof}
For \(j,t\in\{1,\dots,n_m\}\),
\begin{align*}
\langle g_j,g_t\rangle_{\mathcal H}
&=
\frac{1}{\sqrt{\lambda_j\lambda_t}}
\sum_{i=1}^m\sum_{s=1}^m
U_{ij}U_{st}\langle \phi_i,\phi_s\rangle_{\mathcal H}=
\frac{1}{\sqrt{\lambda_j\lambda_t}}
\sum_{i=1}^m\sum_{s=1}^m
U_{ij}U_{st}k(z_i,z_s)\\
&=
\frac{1}{\sqrt{\lambda_j\lambda_t}}
\,(U^\top G U)_{jt}
=
\frac{1}{\sqrt{\lambda_j\lambda_t}}
\,\Lambda_{jt}
=
\delta_{jt}.
\end{align*}
Hence \(g_1,\dots,g_{n_m}\) are orthonormal.
Moreover, each \(g_j\) is a linear combination of \(\phi_1,\dots,\phi_m\), so
\[
\operatorname{span}\{g_1,\dots,g_{n_m}\}\subset V_{Z_m}.
\]
Since \(U\) is orthogonal, from \eqref{eq:orthonormal_basis_V_m}, we get
\[
\phi_i=\sum_{j=1}^{n_m} \sqrt{\lambda_j}\,U_{ij}\,g_j,
\qquad i=1,\dots,m.
\]
Hence \(\phi_1,\dots,\phi_m\in\operatorname{span}\{g_1,\dots,g_{n_m}\}\), and therefore
$
V_{Z_m}=\operatorname{span}\{g_1,\dots,g_{n_m}\}. 
$
\end{proof}

\subsection{Mat\'ern kernels via kernel sections}
\begin{lemma}
\label{lem:matern_zero_regimes}
The kernel pseudometric
\[
\varrho_{\nu,\ell}(x,y)
=
\sqrt{2\bigl(\kappa_{\nu,\ell}(0)
-\kappa_{\nu,\ell}(\|x-y\|_2)\bigr)}
\]
generated by the Mat\'ern kernel \eqref{eq:Mathern_kernel} satisfies
\[
\varrho_{\nu,\ell}(x,y)\le
\begin{cases}
C\|x-y\|_2^\nu, & 0<\nu<1,\\[1mm]
C\|x-y\|_2\sqrt{1+\bigl|\ln\|x-y\|_2\bigr|}, & \nu=1,\\[1mm]
C\|x-y\|_2, & \nu>1,
\end{cases}
\]
where the middle expression is understood as \(0\) when \(x=y\).
\end{lemma}
\begin{proof}
Put
$h=\|x-y\|_2$, $a=\sqrt{2\nu}/\ell$, $z=ah$. Then
\[
\kappa_{\nu,\ell}(h)
=
\frac{2^{1-\nu}}{\Gamma(\nu)}\,z^\nu K_\nu(z).
\]

Assume first that \(\nu\notin\mathbb N\). By (9.6.2) and (9.6.10) of \cite{Abramowitz1972handbook},
\[
K_\nu(z)=\frac{\pi}{2\sin(\pi\nu)}\bigl(I_{-\nu}(z)-I_\nu(z)\bigr),
\qquad
I_{\pm \nu}(z)
=
\left(\frac z2\right)^{\pm \nu}
\left(c_{\pm \nu}+O(z^2)\right),
\qquad z\downarrow0.
\]
Hence
\[
z^\nu I_{-\nu}(z)=c'_\nu+O(z^2),
\qquad
z^\nu I_\nu(z)=c''_\nu z^{2\nu}+O(z^{2\nu+2}),
\]
and therefore
\[
z^\nu K_\nu(z)=c_\nu+O(z^{2\nu})+O(z^2).
\]
It follows that
\[
z^\nu K_\nu(z)=
\begin{cases}
c_\nu+O(z^{2\nu}), & 0<\nu<1,\\[1mm]
c_\nu+O(z^2), & \nu>1,\ \nu\notin\mathbb N.
\end{cases}
\]

Now let \(\nu=n\in\mathbb N\). By (9.6.11) of
\cite{Abramowitz1972handbook}, for \(z>0\),
\begin{equation}
\label{eq:Kn_expansion}
K_n(z)
=
z^{-n}P_{n-1}(z^2)
+
(\ln z)I_n(z)
+
z^nH_n(z^2),
\end{equation}
where \(P_{n-1}\) is a polynomial of degree \(n-1\) and \(H_n\) is analytic
near \(0\).
For \(n\ge2\), since $I_n=O(z^n)$, we get
\[
z^nK_n(z)=c_n+O(z^2)
\]
with some constant \(c_n\).

Thus, for $\nu\neq 1$,
\[
z^\nu K_\nu(z)=
\begin{cases}
c_\nu+O(z^{2\nu}), & 0<\nu<1,\\[1mm]
c_\nu+O(z^2), & \nu>1.
\end{cases}
\]
Since \(\kappa_{\nu,\ell}(0)=1\), this yields
\[
\kappa_{\nu,\ell}(0)-\kappa_{\nu,\ell}(h)
=
\begin{cases}
O(h^{2\nu}), & 0<\nu<1,\\[1mm]
O(h^2), & \nu>1.
\end{cases}
\]
The bound for \(\varrho_{\nu,\ell}\) follows from the representation
\[
\varrho_{\nu,\ell}(x,y)^2
=
2\bigl(\kappa_{\nu,\ell}(0)-\kappa_{\nu,\ell}(\|x-y\|_2)\bigr). 
\]

For \(\nu=1\), the expansion \eqref{eq:Kn_expansion} gives
\[
        zK_1(z)=1+O\!\left(z^2|\ln z|\right),
        \qquad z\downarrow 0,        
\]
since $P_0(0)=1$. Consequently,
$
        \kappa_{1,\ell}(0)-\kappa_{1,\ell}(h)
        =
        O\!\left(h^2(1+|\ln h|)\right). 
$
Hence
\[
\varrho_{1,\ell}(x,y)
=
O\!\left(
h\sqrt{1+|\ln h|}
\right),
\qquad h=\|x-y\|_2.
\]

The estimates above are local at \(h=0\). Their extension to all
\(x,y\in X\) follows from the continuity of \(\varrho_{\nu,\ell}\)
and compactness of \(X\).
\end{proof}

We claim that
\begin{equation}
\label{eq:matern_approximation_rate_pseudometric}
\mathcal E(V_{Z_m})
=
\begin{cases}
O(m^{-\nu/d}), & 0<\nu<1,\\[1mm]
O\!\left(m^{-1/d}\sqrt{\ln(em)}\right), & \nu=1,\\[1mm]
O(m^{-1/d}), & \nu>1.
\end{cases}
\end{equation}
Indeed, by Lemma~\ref{lem:matern_zero_regimes} and
Lemma~\ref{lem:rho_holder_covering}, for every \(m\ge1\) there exists a set
\(Z_m\subset X\), \(|Z_m|=m\), such that
\[
\mathcal E(V_{Z_m})=
\begin{cases}
O(m^{-\nu/d}), & 0<\nu<1,\\[1mm]
O(m^{-1/d}), & \nu>1.
\end{cases}
\]
For \(\nu=1\), by \eqref{eq:covering_number_bound_Euclidean}, we may choose
a Euclidean \(\delta_m\)-net \(Z_m\subset X\), with \(|Z_m|=m\) and
\(\delta_m=O(m^{-1/d})\).
For every \(x\in X\), there exists \(z\in Z_m\) such that
\(\|x-z\|_2\le\delta_m\). Hence, by
Lemma~\ref{lem:matern_zero_regimes},
\[
\varrho_{1,\ell}(x,z)
\le
C\delta_m\sqrt{1+|\ln\delta_m|}
=
O\!\left(m^{-1/d}\sqrt{\ln(em)}\right),
\]
where we used that the function
\(h\mapsto h\sqrt{1+|\ln h|}\)
is increasing for all sufficiently small \(h>0\).
Lemma~\ref{lem:kernel-section_based_cover} therefore yields
\[
\mathcal E(V_{Z_m})
=
O\!\left(m^{-1/d}\sqrt{\ln(em)}\right).
\]

For \(\nu>1\), however, the bound \eqref{eq:matern_approximation_rate_pseudometric} can be substantially improved by exploiting
the Sobolev characterization of the Mat\'ern RKHS together with sharper estimates
for kernel section spaces. Here we follow the ideas of \cite{Zadorozhnyi2021Online}.

Recall that for \(s>0\) the Sobolev space \(H^s(\mathbb R^d)\) can be
defined by means of the Fourier transform:
\[
H^s(\mathbb R^d)
=
\left\{
f\in L_2(\mathbb R^d):
\int_{\mathbb R^d} (1+\|\omega\|_2^2)^s |\widehat f(\omega)|^2\,d\omega<\infty
\right\},\qquad \widehat f(\omega)
=
\frac{1}{(2\pi)^{d/2}}
\int_{\mathbb R^d} f(x)e^{-i x\cdot \omega}\,dx
\]
equipped with the corresponding norm. 
Following \cite{Agranovich2015Sobolev}, we define the Sobolev space on \(X\) as the restriction space
\[
H^s(X)
=
\{\,v|_X:\ v\in H^s(\mathbb R^d)\,\},
\qquad
\|u\|_{H^s(X)}
=
\inf\bigl\{\|v\|_{H^s(\mathbb R^d)}:\ v|_X=u\bigr\}.
\]
Assume that the boundary of $X$ is Lipschitz. Then the
space \(H^s(X)\) defined above coincides with the intrinsic Sobolev--Slobodetskii
space, and the corresponding norms are equivalent; see
\cite{Agranovich2015Sobolev} (Theorem~5.1.1 and Section~14.5).

Now let \(k(x,y)=\Phi(x-y)\) be a continuous shift-invariant kernel on \(\mathbb R^d\).
By \cite[Theorem~2.1]{Kanagawa2025Gaussian},
\[
H_k
=
\left\{
f\in L_2(\mathbb R^d)\cap C(\mathbb R^d):
\|f\|_{H_k}^2
=
\frac{1}{(2\pi)^{d/2}}
\int_{\mathbb R^d}\frac{|\widehat f(\omega)|^2}{\widehat\Phi(\omega)}\,d\omega
<\infty
\right\}.
\]
For the Mat\'ern kernel, the explicit form of
\(\widehat{\kappa}_{\nu,\ell}\) in
\eqref{eq:Matern_Fourier_asymptotics} implies that the RKHS generated by
\(\kappa_{\nu,\ell}\) on \(\mathbb R^d\) coincides with
\(H^{\nu+d/2}(\mathbb R^d)\) as a set, and that the corresponding norms are
equivalent; see also \cite[Example~2.7]{Kanagawa2025Gaussian}.
Let \(\kappa_{\nu,\ell}|_{X\times X}\) be the restriction of the Mat\'ern kernel to
\(X\times X\). By the restriction theorem, see \cite[Corollary 5.8]{Paulsen2016Introduction}, the RKHS generated by \(\kappa_{\nu,\ell}|_{X\times X}\)
is norm-equivalent to \(H^{\nu+d/2}(X)\).

For a finite set \(Z\subset X\) define
\[
h_{Z,X}=\sup_{x\in X}\min_{z\in Z}\|x-z\|_2.
\]
Then \(h_{Z,X}\le \delta\) if and only if \(Z\) is a \(\delta\)-net of \(X\).
Hence, by the definition of the covering number \(N(\delta,X,\|\cdot\|_2)\), there
exists \(Z^\delta\subset X\) such that
\[
|Z^\delta|=N(\delta,X,\|\cdot\|_2),\qquad h_{Z^\delta,X}\le \delta.
\]
If \(X\subset B_r\), then
\[
N(\delta,X,\|\cdot\|_2)
\le
\left(1+\frac{2r}{\delta}\right)^d
\]
by \eqref{eq:covering_number_bound_Euclidean}. Therefore, for every \(m\ge 2\), choosing
\[
\delta_m=\frac{2r}{m^{1/d}-1},
\]
we obtain
$
N(\delta_m,X,\|\cdot\|_2)\le m.
$
Hence there exists a set \(Z_m\subset X\) with \(|Z_m|\le m\) and a constant $C_X$ such that
\begin{equation}
  \label{eq:fill_distance_inequality}
h_{Z_m,X}\le \delta_m  \le C_X m^{-1/d}.
\end{equation}
Adding arbitrary points of \(X\) if necessary, we may assume \(|Z_m|=m\).

Assume now that \(\nu>1\). We want to control the error of the orthogonal projection $P_Z:=\Pi_{V_Z}$ onto the section
space \(V_Z=\operatorname{span}\{\kappa_{\nu,\ell}(\cdot,z):z\in Z\}\), where $Z$ is a finite set. 
For any $f\in H^{\nu+d/2}(X)$ we have
\[ u(z):=f(z)-P_Z f(z)=\langle f-P_Z f,\kappa_{\nu,\ell}(\cdot,z)\rangle_{\mathcal H_{\nu,\ell}}=0.\]
By \cite[Theorem 1.1]{Narcowich2005Sobolev} (with $p=2$, $q=\infty$, $\alpha=0$) we get
\[
\|u\|_{L^\infty(X)}\le C h_{Z,X}^{k+\sigma-d/2}\|u \|_{H^{k+\sigma}(X)},\]
where $\sigma\in(0,1]$, $k>d/2$, $k\in\mathbb N$. Set $s=\nu+d/2$. Since $\nu>1$ we can put 
$k=\lfloor s\rfloor$, $\sigma=s-\lfloor s\rfloor$ if \(s\not\in \mathbb N\)
and
\(k=s-1\), \(\sigma=1\) if \(s\in\mathbb N\). Then
\begin{equation}
  \label{eq:projection_inequality_Narcowich}
\|u\|_{L^\infty(X)}\le C h_{Z,X}^{\nu}\|u \|_{H^{\nu+d/2}(X)}.
\end{equation}

Set $A=I-P_Z$. We have
\[
\|A\kappa_{\nu,\ell}(\cdot,x)\|_{\mathcal H_{\nu,\ell}}
=
\sup_{\|f\|_{\mathcal H_{\nu,\ell}}\le 1}
\bigl|
\langle f,A\kappa_{\nu,\ell}(\cdot,x)\rangle_{\mathcal H_{\nu,\ell}}
\bigr|.
\]
But since \(A\) is self-adjoint in $\mathcal H_{\nu,\ell}$, we have
\[
\langle f,A\kappa_{\nu,\ell}(\cdot,x)\rangle_{\mathcal H_{\nu,\ell}}
=
\langle A f,\kappa_{\nu,\ell}(\cdot,x)\rangle_{\mathcal H_{\nu,\ell}}=Af(x),
\]
and therefore,
\[
\|A\kappa_{\nu,\ell}(\cdot,x)\|_{\mathcal H_{\nu,\ell}}
\le
\sup_{\|f\|_{\mathcal H_{\nu,\ell}}\le 1}\|Af\|_{L^\infty(X)}
=
\sup_{\|f\|_{\mathcal H_{\nu,\ell}}\le 1}\|f-P_Z f\|_{L^\infty(X)}.
\]
Furthermore, by \eqref{eq:projection_inequality_Narcowich} and by the
equivalence of the norms in \(H^{\nu+d/2}(X)\) and \(\mathcal H_{\nu,\ell}\),
\[
 \|f-P_Z f\|_{L^\infty(X)}
 \le C\,h_{Z,X}^{\nu}\|f-P_Z f\|_{H^{\nu+d/2}(X)}
 \le C\,h_{Z,X}^{\nu}\|f\|_{H^{\nu+d/2}(X)}.
\]

From the last two inequalities it follows that 
\[\mathcal E(V_Z)
=
\sup_{x\in X}\|(I-P_Z)\kappa_{\nu,\ell}(\cdot,x)\|_{\mathcal H_{\nu,\ell}}\le C\,h_{Z,X}^{\nu}.\]
Applying \eqref{eq:fill_distance_inequality} we obtain
$\mathcal E(V_{Z_m}) \le C\,m^{-\nu/d}$
for $\nu>1$. Combining this estimate with
\eqref{eq:matern_approximation_rate_pseudometric}, we obtain
\begin{equation}
\label{eq:matern_approximation_rate}
\mathcal E(V_{Z_m})
=
\begin{cases}
O(m^{-\nu/d}), & \nu\neq1,\\[1mm]
O\!\left(m^{-1/d}\sqrt{\ln(em)}\right), & \nu=1.
\end{cases}
\end{equation}

By Lemma~\ref{lem:gram_orthonormal_basis}, there exists
an orthonormal basis \(g_1,\dots,g_{n_m}\) of \(V_{Z_m}\), where \(n_m:=\dim V_{Z_m}\le m\). Define the associated feature map
\[
\Psi_{Z_m}(x)=(g_1(x),\dots,g_{n_m}(x))\in\mathbb R^{n_m}.
\]
Applying Corollary~\ref{cor:rough_bound_subspace} to the subspace \(V_{Z_m}\), we obtain
\[
R_T^{\mathcal A(\lambda)}(f_{1:T})
=
O\!\left(
(\ln T)^2
+
\sqrt{n_mT\,P_T^{\mathcal H}(f_{1:T})}
+
n_m\ln\!\left(1+\frac{T}{n_m}\right)
+
T\,\mathcal E(V_{Z_m})
\right).
\]
Since \(n_m\le m\), and the terms containing $n_m$ are increasing, we obtain a bound of the same form as in Corollary~\ref{cor:rough_bound_subspace}:
\begin{equation}
  \label{eq:kernel_section_based_bound}
R_T^{\mathcal A(\lambda)}(f_{1:T})
=
O\!\left(
(\ln T)^2
+
\sqrt{mT\,P_T^{\mathcal H}(f_{1:T})}
+
m\ln\!\left(1+\frac{T}{m}\right)
+
T\,\mathcal E(V_{Z_m})
\right).
\end{equation}

For \(\nu\neq1\), the inequality
\eqref{eq:matern_approximation_rate} corresponds to the slow approximation
regime with \(\beta=\nu/d\). Although the dimension of the subspace
\(V_{Z_m}\) may be smaller than \(m\), the proof of
Theorem~\ref{thm:adaptive_over_m_slow_general} applies verbatim,
because it only uses the bound \eqref{eq:kernel_section_based_bound} of
Corollary~\ref{cor:rough_bound_subspace}, and does not rely on the identity
\(\dim V_{Z_m}=m\).
Applying Theorem~\ref{thm:adaptive_over_m_slow_general} with
\(\beta=\nu/d\), we conclude that the predictor \(\overline{\mathcal A}\)
obtained by aggregating the VE-DVAW forecasters corresponding to the feature
maps \(\Psi_{Z_m}\), \(m\in\mathcal M_{T,\nu/d}\), satisfies
\begin{equation}
\label{eq:matern_bound}
R_T^{\overline{\mathcal A}}(f_{1:T})
=
O\!\left(
T^{\frac{\nu+d}{2\nu+d}}
\bigl(P_T^{\mathcal H_{\nu,\ell}}(f_{1:T})\bigr)^{\frac{\nu}{2\nu+d}}
+
T^{\frac{d}{\nu+d}}(\ln T)^{\frac{\nu}{\nu+d}}
\right),
\qquad \nu\neq1.
\end{equation}

For the borderline case \(\nu=1\), put
\(P=P_T^{\mathcal H_{1,\ell}}(f_{1:T})\) and \(L=\ln(eT)\).
Using \eqref{eq:matern_approximation_rate} in
\eqref{eq:kernel_section_based_bound} and repeating the aggregation argument
from the proof of Theorem~\ref{thm:adaptive_over_m_slow_general} over the grid
\(\mathcal M_{T,1/d}\), we obtain the bound
\[
R_T^{\overline{\mathcal A}}(f_{1:T})
=
O\!\left(
(\ln T)^2+
\min_{m\in\mathcal M_{T,1/d}}\Phi_P(m)
\right),
\]
similar to \eqref{eq:intermediate_bound} but with
\[
\Phi_P(m)
=
\sqrt{mTP}
+
m\ln\!\left(1+\frac{T}{m}\right)
+
Tm^{-1/d}\sqrt{\ln(em)}.
\]

To compare the dyadic grid \(\mathcal M_{T,1/d}\) with the
continuous range, put
\(M_{T,1/d}:=T^{d/(d+1)}\).
For every \(m\in[1,M_{T,1/d}]\), putting
\(\underline m=2^{\lfloor\log_2m\rfloor}\in\mathcal M_{T,1/d}\), we have
\[
\min_{m'\in\mathcal M_{T,1/d}}\Phi_P(m')
\le
\Phi_P(\underline m)
\le
2^{1/d}\Phi_P(m).
\]
Indeed,
\[
\underline m^{-1/d}\sqrt{\ln(e\underline m)}
\le
2^{1/d}m^{-1/d}\sqrt{\ln(em)},
\]
whereas the first two terms in \(\Phi_P(m)\) are increasing in \(m\).
Hence
\[
\min_{m\in\mathcal M_{T,1/d}}\Phi_P(m)
=
O\!\left(
\inf_{1\le m\le M_{T,1/d}}\Phi_P(m)
\right).
\]
Furthermore, since \(m\le M_{T,1/d}\le T\), we get
\begin{equation}
\label{eq:intermediate_borderline_bound}
\Phi_P(m)
=
O\!\left(
\sqrt{mTP}+mL+Tm^{-1/d}L^{1/2}
\right),
\qquad
1\le m\le M_{T,1/d}.
\end{equation}

Put
$
P_0
=
T^{-\frac1{d+1}}
L^{\frac{3d+4}{2(d+1)}}.
$
If \(P\ge P_0\), choose
\[
m_d=\left(\frac{TL}{P}\right)^{\frac d{d+2}}.
\]
As in the proof of Theorem~\ref{thm:adaptive_over_m_slow_general},
we have \(P\le 2R_fT\), and hence \(m_d\ge1\) for all sufficiently large
\(T\). Moreover, since \(P\ge P_0\),
\[
m_d
\le
\left(\frac{TL}{P_0}\right)^{\frac d{d+2}}
=
T^{\frac d{d+1}}
L^{-\frac{d}{2(d+1)}}
=
M_{T,1/d}L^{-\frac{d}{2(d+1)}}
\le
M_{T,1/d}.
\]
Thus \(m_d\in[1,M_{T,1/d}]\). The choice \(m=m_d\) balances the first and
third terms in the bound \eqref{eq:intermediate_borderline_bound}:
\[
\sqrt{m_dTP}
=
Tm_d^{-1/d}L^{1/2}
=
T^{\frac{d+1}{d+2}}
P^{\frac1{d+2}}
L^{\frac{d}{2(d+2)}},
\]
and the middle term \(m_dL\) is of no larger order by the definition of
\(P_0\). Thus
\[
\inf_{1\le m\le M_{T,1/d}}\Phi_P(m)
=
O\!\left(
T^{\frac{d+1}{d+2}}
P^{\frac1{d+2}}
L^{\frac{d}{2(d+2)}}
\right).
\]

If \(P\le P_0\), choose
\[
m_s
=
T^{\frac d{d+1}}
L^{-\frac{d}{2(d+1)}}
=
M_{T,1/d}L^{-\frac{d}{2(d+1)}}.
\]
Clearly, \(m_s\in[1,M_{T,1/d}]\) for all sufficiently large \(T\).
This choice balances the second and third terms in \eqref{eq:intermediate_borderline_bound}:
\[
m_sL
=
Tm_s^{-1/d}L^{1/2}
=
T^{\frac d{d+1}}
L^{\frac{d+2}{2(d+1)}}.
\]
Moreover, since \(P\le P_0\),
\[
\begin{aligned}
\sqrt{m_sTP}
&=
T^{\frac{2d+1}{2(d+1)}}
L^{-\frac{d}{4(d+1)}}P^{1/2}
\le
T^{\frac{2d+1}{2(d+1)}}
L^{-\frac{d}{4(d+1)}}
T^{-\frac1{2(d+1)}}
L^{\frac{3d+4}{4(d+1)}}=
T^{\frac d{d+1}}
L^{\frac{d+2}{2(d+1)}}.
\end{aligned}
\]
Thus
\[
\inf_{1\le m\le M_{T,1/d}}\Phi_P(m)
\le
\Phi_P(m_s)
=
O\!\left(
T^{\frac d{d+1}}
L^{\frac{d+2}{2(d+1)}}
\right).
\]

Since \(L\asymp\ln T\) and
$
(\ln T)^2
=
O\!\left(
T^{\frac d{d+1}}
(\ln T)^{\frac{d+2}{2(d+1)}}
\right),
$
the dyadic-grid comparison above and the two cases considered yield
\begin{equation}
\label{eq:matern_bound_nu1}
R_T^{\overline{\mathcal A}}(f_{1:T})
=
O\!\left(
T^{\frac{d+1}{d+2}}
\bigl(P_T^{\mathcal H_{1,\ell}}(f_{1:T})\bigr)^{\frac1{d+2}}
(\ln T)^{\frac{d}{2(d+2)}}
+
T^{\frac d{d+1}}
(\ln T)^{\frac{d+2}{2(d+1)}}
\right).
\end{equation}

Note that in the smooth regime \(\nu>1\) the Mercer truncation scheme gives by \eqref{eq:epsilon_m_poly_decay} and \eqref{eq:beta_Matern_Mercer},
\[
\mathcal E(V_m)
=
O\!\left(
m^{-\alpha_{\nu,p}}
\right),
\qquad
\alpha_{\nu,p}
:=
\frac{p(d+4\nu)}{2d(d+2p)},
\qquad 1\le p<\nu.
\]
But the section-space approximation rate \eqref{eq:matern_approximation_rate}
is strictly better, since
\[
\frac{\nu}{d}-\alpha_{\nu,p}
=
\frac{2\nu-p}{2(d+2p)}>0.
\]

\begin{remark}
In the static case, where $f_t=f$ and
$P_T^{\mathcal H_{\nu,\ell}}(f_{1:T})=0$, Theorem~4 of
\cite{Zadorozhnyi2021Online} gives, for the KAAR algorithm, the bound
\[
R_T(f)
=
O\!\left(
T^{\frac{d}{2s+d}+\varepsilon}\ln T
\right),
\qquad
\|f\|_{H^s(X)}\le R,
\qquad
s>\frac d2
\]
for any $\varepsilon>0$.
Since the Mat\'ern RKHS $\mathcal H_{\nu,\ell}$ is norm-equivalent to
$H^{\nu+d/2}(X)$, substituting $s=\nu+d/2$ gives
\[
R_T(f)
=
O\!\left(
T^{\frac{d}{2(\nu+d)}+\varepsilon}\ln T
\right).
\]
This static rate is better than the static specialization of
\eqref{eq:matern_bound}:
$R_T(f)
=
O\!\left(
T^{\frac{d}{\nu+d}}
(\ln T)^{\frac{\nu}{\nu+d}}
\right).
$
It remains an open question whether the regret bound \eqref{eq:matern_bound} can be improved by an appropriate generalization of KAAR or by other means.
\end{remark}

\section{Conclusion}
We proposed a finite-dimensional approximation approach to
dynamic regret minimization for online regression in RKHS under the square loss.
The construction combines discounted VAW forecasters with VAW-based aggregation
over discount parameters and, in the slow approximation regime, over subspace
dimensions. The resulting bounds are governed by the RKHS path length of the
comparator sequence and by the approximation quality of the chosen
finite-dimensional subspaces.

The framework was applied to several approximation schemes. Explicit feature
expansions lead to fast-regime bounds for Gaussian and analytic dot-product
kernels, while Mercer truncations provide a general spectral construction.
However, explicit orthonormal expansions are available only in special cases,
and Mercer truncations are difficult to implement and seem to lead to
suboptimal estimates. Subspaces spanned by kernel sections provide a more
universal alternative, since they require only kernel evaluations and can be
represented through Gram matrices. For Mat\'ern kernels this scheme also fits
naturally with the Sobolev description of the RKHS, although its computational
efficiency requires a separate study.

A further direction is to bring the present dynamic-regret analysis closer to
the kernelized KAAR framework and its projected variants.
The results of \cite{Jezequel2019efficient,Zadorozhnyi2021Online} show that
KAAR-type algorithms can exploit effective dimension and projection arguments
to obtain sharp static regret bounds for important RKHS and Sobolev classes.
It would be interesting to understand whether analogous ideas can be combined
with discounting and path-length-dependent analysis in the dynamic setting.

\printbibliography
\end{document}